\documentclass{article}
\usepackage[utf8]{inputenc}
\usepackage[preprint]{neurips_2021}

\usepackage{mathtools}
\usepackage{amssymb,amsmath,amsthm}
\usepackage{xcolor}

\usepackage[Algorithm,ruled]{algorithm}
\usepackage{algorithmic}
\usepackage{enumitem}
\usepackage{mathtools} 
\usepackage{fullpage}
\usepackage{natbib}
\usepackage{bm}
\usepackage{multirow}
\usepackage{array}
\usepackage{float}

\usepackage{hyperref}
\usepackage[capitalise]{cleveref}
\Crefname{assumption}{Assumption}{Assumptions}

\usepackage{autonum}

\newtheorem{theorem}{Theorem}
\newtheorem{theorem*}{Theorem}
\newtheorem{corollary}{Corollary}
\newtheorem{assumption}{Assumption}

\newtheorem{remark}{Remark}

\newtheorem{lemma}{Lemma}

\newcommand{\Ind}{\textbf{1}}
\newcommand{\rf}{\mathrm{ref}}

\newcommand{\argmax}{\mathop{\arg\max}}
\newcommand{\argmin}{\mathop{\arg\min}}

\DeclareMathOperator{\Var}{Var}

\newcommand{\A}{\mathcal A}
\newcommand{\X}{\mathcal X}
\newcommand{\Y}{\mathcal Y}

\newcommand{\ie}{\emph{i.e.}}
\newcommand{\eg}{\emph{e.g.}}

\title{
Post-Contextual-Bandit Inference
}

\author{%
  Aurélien Bibaut\thanks{Alphabetical order} \\
  Netflix
\And
    Antoine Chambaz\footnotemark[1]\\
    Université Paris Descartes
\And
    Maria Dimakopoulou\footnotemark[1]\\
    Netflix
\And
    Nathan Kallus\footnotemark[1]\\
    Cornell University and Netflix
\And 
Mark van der Laan\footnotemark[1]\\
University of California, Berkeley
}

\begin{document}

\maketitle

\begin{abstract}
Contextual bandit algorithms are increasingly replacing non-adaptive A/B tests in e-commerce, healthcare, and policymaking because they can both improve outcomes for study participants and increase the chance of identifying good or even best policies. To support credible inference on novel interventions at the end of the study, nonetheless, we still want to construct valid confidence intervals on average treatment effects, subgroup effects, or value of new policies. The adaptive nature of the data collected by contextual bandit algorithms, however, makes this difficult: standard estimators are no longer asymptotically normally distributed and classic confidence intervals fail to provide correct coverage. While this has been addressed in non-contextual settings by using stabilized estimators, the contextual setting poses unique challenges that we tackle for the first time in this paper. We propose the Contextual Adaptive Doubly Robust (CADR) estimator, the first estimator for policy value that is asymptotically normal under contextual adaptive data collection. The main technical challenge in constructing CADR is designing adaptive and consistent conditional standard deviation estimators for stabilization. Extensive numerical experiments using 57 OpenML datasets demonstrate that confidence intervals based on CADR uniquely provide correct coverage.
\end{abstract}

\section{Introduction}

Contextual bandits, where personalized decisions are made sequentially and simultaneously with data collection, are increasingly used to address important decision-making problems where data is limited and/or expensive to collect, with applications in product recommendation \citep{li2010contextual}, revenue management \citep{kallus2020dynamic,qiang2016dynamic}, and personalized medicine \citep{tewari2017ads}. Adaptive experiments, whether based on bandit algorithms or Bayesian optimization, are increasingly being considered in place of classic randomized trials in order to improve both the outcomes for study participants and the chance of identifying the best treatment allocations \citep{atheytrial,kasytrial,kasy2021adaptive,bakshy2018ae}.

But, at the end of the study, we still want to construct valid confidence intervals on average treatment effects, subgroup effects, or the value of new personalized interventions. Such confidence intervals are, for example, crucial for enabling credible inference on the presence or absence of improvement of novel policies. However, due to the adaptive nature of the data collection, unlike classic randomized trials, standard estimates and their confidence intervals actually fail to provide correct coverage, that is, contain the true parameter with the desired confidence probability (\eg, 95\%). A variety of recent work has recognized this and offered remedies \citep{hadad2019confidence,luedtke_vdL2016}, but only for the case of non-contextual adaptive data collection. Like classic confidence intervals, when data comes from a contextual bandit -- or any other context-dependent adaptive data collection -- these intervals also fail to provide correct coverage. In this paper, we propose the first asymptotically normal estimator for the value of a (possibly contextual) policy from \emph{context-dependent} adaptively collected data. This asymptotic normality leads directly to the construction of valid confidence intervals.

Our estimator takes the form of a \emph{stabilized} doubly robust estimator, that is, a weighted time average of an estimate of the so-called canonical gradient using plug in estimators for the outcome model, where each time point is inversely weighted by its estimated conditional standard deviation given the past. We term this the Contextual Adaptive Doubly Robust (CADR) estimator. We show that, given consistent conditional variance estimates which at each time point only depend on previous data, the CADR estimator is asymptotically normal, and as a result we can easily construct asymptotically valid confidence intervals. This normality is in fact robust to misspecifying the outcome model. A significant technical challenge is actually constructing such variance estimators. We resolve this using an adaptive variance estimator based on the importance-sampling ratio of current to past (adaptive) policies at each time point. We also show that we can reliably estimate outcome models from the adaptively-collected data so that we can plug them in. Extensive experiments using 57 OpenML datasets demonstrate the failure of previous approaches and the success of ours at constructing confidence intervals with correct coverage.






\subsection{Problem Statement and Notation}


\subparagraph{The data.} Our data consists of a sequence of observations indexed $t=1,\dots,T$ comprising of context $X(t)\in\X$, action $A(t)\in\A$, and outcome $Y(t)\in\Y\subset\mathbb R$ generated by an adaptive experiment, such as a contextual bandit algorithm.
Roughly, at each round $t=1,2,\dots,T$, an agent formed a contextual policy $g_t(a\mid x)$ based on all past observations, then observed an independently drawn context vector $X(t)\sim Q_{0,X}$, carried out an action $A(t)$ drawn from its current policy $g_t(\cdot\mid X(t))$, and observed an outcome $Y(t)\sim Q_{0,Y}(\cdot\mid A(t),X(t))$ depending only on the present context and action. 
The action and context measurable spaces $\X,\A$ are arbitrary, \eg, finite or continuous.

More formally, we let $O(t):=(X(t), A(t), Y(t))$ and make the following assumptions about the sequence $O(1),\dots,O(T)$ comprising our dataset. First, we assume $X(t)$ is independent of all else given $A(t)$ and has a time-independent marginal distribution that we denote by $Q_{0,X}$. Second, we assume $A(t)$ is independent of all else given $O(1),\dots,O(t-1),X(t)$ and we set $g_t(\cdot\mid X(t))$ to its (random) conditional distribution given $O(1),\dots,O(t-1),X(t)$. Third, we assume $Y(t)$ is independent of all else given $X(t),A(t)$ and has a time-independent conditional distribution given $X(t)=x,A(t)=a$ that is denoted by $Q_{0,Y}(\cdot\mid A,X)$. 
The distributions $Q_{0,X}$ and $Q_{0,Y}$ are unknown, while the policies $g_t(a\mid x)$ are known, as would be the case when running an adaptive experiment.
To simplify presentation we endow $\A$ with a base measure $\mu_\A$ (\eg, counting for finite actions or Lebesgue for continuous actions) and identify policies $g_t$ with conditional densities with respect to (w.r.t.) $\mu_\A$. In the case of $K<\infty$ actions, policies are maps from $\X$ to the $K$-simplex.

Note that, as the agent updates its policy based on already collected observations, $g_t$ is a random $O(1),\ldots,O(t-1)$-measurable object. This is the major departure from the setting considered in other literature on off-policy evaluation, which only consider a fixed logging policy, $g_t=g$, that is independent of the data. See \cref{sec:litreview}.


\subparagraph{The target parameter.} 
We are interested in inference on a \emph{generalized average causal effect} expressed as a functional of the unknown distributions above, $\Psi_0=\Psi(Q_{0,X},Q_{0,Y})$, where for any distributions $Q_X,Q_{Y}$, we define 
\begin{align}\notag
    \Psi(Q_{X},Q_{Y}) := \int y Q_X(dx)g(a\mid x)d\mu_\A(a)Q_Y(dy\mid a,x),
\end{align}
where $g^*(a\mid x):\A\times\X\to[-G,G]$ is a given fixed, bounded function.
Two examples are: (a) when $g^*$ is a policy (conditional density), then $\Psi_0$ is its value; (b) when $g^*$ is the difference between two policies then $\Psi_0$ is the difference between their values. A prominent example of the latter is when $\A=\{+1,-1\}$ and $g^*(a\mid x)=a$, which is known as the average treatment effect. If we include an indicator for $x$ being in some set, then we get the subgroup effect.
%
%

Defining the conditional mean outcome,
$$
\bar{Q}_0(a,x):=E_{Q_{0,Y}(\cdot\mid x,a)}[Y]=\int yQ_{0,Y}(dy\mid a,x),
$$
we note that the target parameter only depends on $Q_{0,Y}$ via $\bar{Q}_{0}$, so we also overload notation and write $\Psi(Q_{X},\bar{Q})=\int \bar Q(a,x)Q_X(dx)g(a\mid x)d\mu_\A(a)$ for any function $\bar Q:\A\times \X\to\Y$.
Note that when $|\A|<\infty$ and $\mu_\A$ is the counting measure, the integral over $a$ is a simple sum.

\paragraph{Canonical gradient.} We will make repeated use of the following function:
for any conditional density $(a,x) \mapsto g(a \mid x)$, any probability distribution $Q_X$ over the context space $\mathcal{X}$, and any function $\bar{Q}:\mathcal{A} \times \mathcal{X}$, we define the function $D'(g, \bar{Q}):\mathcal O\to\mathbb R$ by
\begin{align}\notag
    D'(g, \bar{Q})(x,a,y) := \frac{g^*(a \mid x)}{g(a \mid x)}(y - \bar{Q}(a,x)) + \int \bar{Q}(a',x)g^*(a' \mid x) d\mu_\A(a').
\end{align}
Further, define $D(g, Q_X, \bar{Q})=D'(g, Q_X, \bar{Q})-\Psi(Q_X, \bar{Q})$, which coincides with the so-called canonical gradient of the target parameter $\Psi$ w.r.t. the usual nonparametric statistical model comprising all joint distributions over $\mathcal{O}$ \citep{van2000asymptotic,van2003unified}.

\paragraph{Integration operator notation.} 
For any policy $g$ and distributions $Q_X,Q_Y$, denote by $P_{Q,g}$ the induced distribution on $\mathcal{O}$.
For any function $f:\mathcal{O} \rightarrow \mathbb{R}$, we use the integration operator notation
\begin{align}\notag
    P_{Q,g}f = \int f(x,a,y) Q_X(dx) g(a \mid x)d\mu_\A(a) Q_Y(dy\mid a,x),
\end{align}
that is, the expectation w.r.t. $P_{Q,g}$ \emph{alone}.
Then, for example, for any $O(1),\ldots,O(s-1)$-measurable random function $f:\mathcal O\rightarrow\mathbb R$, we have that $P_{Q_0,g_s} f = E_{Q_0,g_s}[f(O(s)) \mid {O}(1),\ldots,O(s-1)]$.



\subsection{Related Literature and Challenges for Post-Contextual-Bandit Inference}\label{sec:litreview}

\paragraph{Off-policy evaluation.}
In non-adaptive settings, where $g_t=g$ is fixed and does not depend on previous observations, common off-the shelf estimators for the mean outcome under $g^*$ include the Inverse Propensity Scoring (IPS) estimator \citep{beygelzimer2009offset,li2011unbiased} and and the Doubly Robust (DR) estimator \citep{dudik2011doubly,robins1994estimation}:
\begin{align}\notag
    \widehat{\Psi}^{\mathrm{IPS}} := \frac{1}{T} \sum_{t=1}^T D'(g,0)
    ,
    \qquad
    \widehat{\Psi}^{\mathrm{DR}} := \frac{1}{T} \sum_{t=1}^T D'(g,\widehat{\bar Q})
\end{align}
where $\widehat{\bar{Q}}$ is an estimator of the outcome model $\bar{Q}_0(a,x)$. If we use cross-fitting to estimate $\widehat{\bar{Q}}$ \citep{chernozhukov2017double}, then both the IPS and DR estimators are unbiased and asymptotically normal, permitting straightforward inference using Wald confidence intervals (\ie, $\pm1.96$ of the estimated standard error).
There also exist many variants of the IPS and DR estimators that, rather than plugging in the importance sampling (IS) ratios $(g^* / g_t)(A(t) \mid X(t))$ and/or outcome-model estimators, instead choose them directly with the aim to minimize error \citep[\eg][]{kallus2018balanced,farajtabar2018more,thomas2016data,wang2017optimal,kallus2019intrinsically}.

\paragraph{Inference challenges in adaptive settings.}
In the adaptive setting, it is easy to see that, if in the $t$th term for DR we use an outcome model $\widehat{\bar{Q}}_{t-1}$ fit using only the observations $O(1),\ldots,O(t-1)$, then both the IPS and DR estimators both remain unbiased. However, neither generally converges to a normal distribution.
One key difference between the non-adaptive and adaptive settings is that the IS ratios $(g^* / g_t)(A(t) \mid X(t))$ can both diverge to infinity or converge to zero. As a result of this, the above two estimators may either be dominated by their first terms or their last terms. At a more theoretical level, this violates the classical condition of martingale central limit theorems that the conditional variance of the terms given previous observations stabilizes asymptotically. 


\paragraph{Stabilized DR estimators in non-contextual settings.} The issue for inference due to instability of the DR estimator terms was recognized by \citet{luedtke_vdL2016} in another setting. They work in the non-adaptive setting but consider the problem of inferring the maximum mean outcome over all policies when the optimal policy is non-unique. Their proposal is a so-called \textit{stabilized estimator}, in which each term is inversely weighted by an estimate of its conditional standard deviation given the previous terms. This stabilization trick has been also been reused for off-policy inference from \emph{non-contextual} bandit data by \citet{hadad2019confidence}, as the stabilized estimator remains asymptotically normal, permitting inference. In their non-contextual setting, an estimate of the conditional standard deviation of the terms can easily be obtained by the inverse square root propensities. In contrast, in our \emph{contextual} setting, obtaining valid stabilization weights is more challenging and requires a construction involving adaptive training on past data.

\subsection{Contributions} In this paper, we construct and analyze a stabilized estimator for policy evaluation from context-dependent adaptively collected data, such as the result of running a contextual bandit algorithm. This then immediately enables inference. After constructing a generic extension of the stabilization trick, the main technical challenge is to construct a sequence of estimators $\widehat{\sigma}_1,\ldots, \widehat{\sigma}_T$ of the conditional standard deviations that are both consistent and such that for each $t$, $\widehat{\sigma}_t$ only uses the previous data points $O(1),\ldots,O(t-1)$. We show in extensive experiments across a large set of contextual bandit environments that our confidence intervals uniquely achieve close to nominal coverage.

\section{Construction and Analysis of the Generic Contextual Stabilized Estimator}

In this section, we give a generic construction of a stabilized estimator in our contextual and adaptive setting. That is, given generic plug-ins for outcome model and conditional standard deviation. We then provide conditions under which the estimator is asymptotically normal, as desired. To develop CADR, we will then proceed to construct appropriate plug in estimators in the proceeding sections.

\subsection{Construction of the Estimator}


\paragraph{Outcome and variance estimators.} Our estimator uses a sequence $(\widehat{\bar{Q}}_t)_{t \geq 1}$ of estimators of the outcome model $\bar{Q}_0$, such that, for every $t$, $\widehat{\bar{Q}}_t$ is $O(1),\ldots,O(t)$-measurable, that is, is trained using \emph{only} the data up to time $t$. 
%
%
A key part of our estimator are the conditional variance estimators.

Additionally, we require estimates of the conditional standard deviation of the canonical gradient.
Define
\begin{align}\notag
\sigma_{0,t}&:= \sigma_{0,t}(g_t),\\\text{where}~~
    \sigma_{0,t}^2(g) &:= \Var_{Q_0,g}\left( D'(g, \widehat{\bar{Q}}_{t-1})(O(t)) \mid O(1),\ldots,O(t-1) \right).
\end{align}
Let $(\widehat{\sigma}_t)_{t \geq 1}$ be a given sequence of estimates of $\sigma_{0,t}$ such that $\widehat{\sigma}_t$ is $O(1),\ldots,O(t-1)$-measurable, that is, is estimated using \emph{only} the data up to time $t$. 

\paragraph{The generic form of the estimator.} The generic contextual stabilized estimator is then defined as:
\begin{align}\label{eq:stabilized-one-step}
    \widehat{\Psi}_T := \left( \frac{1}{T} \sum_{t=1}^T \widehat{\sigma}_t^{-1} \right)^{-1} \frac{1}{T} \sum_{t=1}^T \widehat{\sigma}_t^{-1} D'(g,\widehat{\bar Q}_{t-1})
    .
\end{align}

\subsection{Asymptotic normality guarantees}

We next characterize the asymptotic distribution of $\widehat{\Psi}_T$ under some assumptions.



\begin{assumption}[Non degenerate efficiency bound]\label{assumption:non_degenerate_EB}
$
    \inf_g P_{Q_0,g} D^2(g, \bar{Q}_0, Q_{0,X}) > 0.
$
\end{assumption}
Assumption \ref{assumption:non_degenerate_EB} states that there is no fixed logging policy $g$ such that the efficiency bound for estimation of $\Psi(\bar{Q}_0, Q_{0,X})$ in the nonparametric model, from i.i.d. draws of $P_{Q_0,g}$, is zero. If assumption \ref{assumption:non_degenerate_EB} does not hold, there exists a logging policy $g$ such that, if $O=(X,A,Y) \sim P_{Q_0,g}$, then $(g^*(A\mid X) / g(A \mid X)) Y$ equals $\Psi(\bar{Q}_0, Q_{0,X})$ with probability 1. In other words, if assumption  \ref{assumption:non_degenerate_EB} does not hold, there exists a logging policy $g$ such that $\Psi(\bar{Q}_0, Q_{0,X})$ can be estimated with no error with probability 1 from a single draw of $P_{Q_0,g}$. 
Thus, it is very lax.
An easy sufficient condition for \cref{assumption:non_degenerate_EB} is that the outcome model has nontrivial variance in that $\Var_{Q_{0,X}}(\int \bar{Q}(a,X)g^*(a \mid x) d\mu_\A(a))>0$.

\begin{assumption}[Consistent standard deviation estimators.]\label{assumption:stdev_est_consistency} $\widehat{\sigma}_t - \sigma_{0,t} \xrightarrow{t\to\infty}0$ almost surely.
\end{assumption}
In the next section we will proceed to construct specific estimators $\widehat{\sigma}_t$ that satisfy \cref{assumption:stdev_est_consistency}, leading to our proposed CADR estimator and confidence intervals.

\begin{assumption}[Exploration rate]\label{assumption:exp_rate_generic_stab_one_step} 
For any $t\geq1$, we have that
$\inf_{a \in \mathcal{A},x \in \mathcal{X}} g_t(a \mid x) \gtrsim t^{-1/2}$ almost surely.
\end{assumption}
Here, $a_t\gtrsim b_t$ means that for some constant $c>0$, we have $a_t\geq cb_t$ for all $t\geq1$.
\Cref{assumption:exp_rate_generic_stab_one_step} requires that the exploration rate of the adaptive experiment does not decay too quickly.

Based on these assumptions, we have the following asymptotic normality result:
\begin{theorem}\label{thm:asymp_normal_stab_one_step}
Denote $\Gamma_T := \left( T^{-1} \sum_{t=1}^T \widehat{\sigma}_t^{-1} \right)^{-1}$.
Under \cref{assumption:non_degenerate_EB,assumption:stdev_est_consistency,assumption:exp_rate_generic_stab_one_step}, it holds that 
\begin{align}\notag
    \Gamma_T^{-1} \sqrt{T} \left( \widehat{\Psi}_T - \Psi_0 \right) \xrightarrow{d} \mathcal{N}(0,1).
\end{align}
\end{theorem}

\begin{remark}
Theorem 1 does not require the outcome model estimator to converge at all. As we will see in \cref{sec:variance-estimator}, our conditional variance estimator does require that the outcome model converges to a fixed limit $\bar{Q}_1$, but this limit does not have to be the true outcome model $\bar{Q}_0$. In other words, consistency of the outcome model is not required at any point of our analysis.
\end{remark}

\section{Construction of the Conditional Variance Estimator and CADR}
\label{sec:variance-estimator}

\begin{algorithm}[t!]\caption{The CADR Estimator and Confidence Interval}\label{alg:cadr}
\begin{algorithmic}
\STATE{\textbf{Input:} Data $O(1),\dots,O(T)$, policies $g_1,\dots,g_T$, target $g^*$, outcome regression estimator}
\FOR{$t = 1,2,\dots,T$}
\STATE{Train $\widehat{\bar Q}_{t-1}$ on $O(1),\dots,O(t-1)$ using the outcome regression estimator}
\STATE{Set $D'_{t,s}=D(g_s,\widehat{\bar Q}_{t-1})(O(t))$ for $s=t,\dots,T$\hfill// (note index order compared to next line)} 
\STATE{Set $\widehat\sigma_t^2=\frac1{t-1}\sum_{s=1}^{t-1}\frac{g_t(A(s)\mid X(s))}{g_s(A(s)\mid X(s))}(D'_{s,t})^2-\left(\frac1{t-1}\sum_{s=1}^{t-1}\frac{g_t(A(s)\mid X(s))}{g_s(A(s)\mid X(s))}D'_{s,t}\right)^2$}
\ENDFOR
\STATE{Set $\Gamma_T=\left(\frac1T\sum_{t=1}^T\widehat\sigma_t^{-1}\right)^{-1}$}
\STATE{Return estimate $\widehat \Psi_T=\frac{\Gamma_T}T\sum_{t=1}^T\widehat\sigma_t^{-1}D'_{t,t}$ and confidence intervals $\operatorname{CI}_\alpha=[\widehat\Psi_T\pm \zeta_{1-\alpha/2}\Gamma_T/\sqrt{T}]$}
\end{algorithmic}
\end{algorithm}

We now tackle the construction of $\widehat \sigma_t$ satisfying our assumptions; namely, they must be adaptively trained only on past data at each $t$ and they must be consistent.
Observe that $\sigma^2_{0,t} = \sigma^2_0(g_t,\widehat{\bar{Q}}_{t-1})$, where we define
\begin{align}
\sigma^2_0(g, \bar{Q}) &:= \Phi_{0,1}(g,\bar{Q}) - (\Phi_{0,2}(g,\bar{Q}))^2,\\
    \Phi_{0,i}(g,\bar{Q}) &:= P_{Q_0,g} (D')^i(g,\bar{Q}),\,i=1,2.
\end{align}

Designing an $O(1),\ldots,O(t-1)$-measurable estimator of $\sigma_{0,t}^2$ presents several challenges. First, while we can only use observations $O(1),\ldots,O(t-1)$ to estimate it, $\sigma_{0,t}^2$ is defined as a function of integrals w.r.t. $P_{Q_0,g_t}$, from which we have only one observation, namely $O(t)$. Second, our estimation target $\sigma_{0,t}^2 = \sigma_0(g_t, \widehat{\bar{Q}}_{t-1})$ is \emph{random} as it depends on $g_t$ and $\widehat{\bar{Q}}_t$. Third, $g_t,\widehat{\bar{Q}}_t$ depend on the same observations $O(1),\ldots,O(t-1)$ that we have at our disposal to estimate $\sigma_{0,t}^2$.

\paragraph{Representation via importance sampling.} We can overcome the first difficulty via importance sampling, which allows us to write $\Phi_{0,i}(g,\bar{Q})$, $i=1,2$ as integrals w.r.t. $P_{Q_0,g_s}$, $s=1,\ldots,t-1$, \ie, the conditional distributions of observations $O(s)$, $s=1,\ldots,t-1$ given their respective past. Namely, for any $s \geq 1$, $i=1,2$, we have that
\begin{align}
    \Phi_{0,i}(g,\bar{Q}) = P_{Q_0,g_s} \frac{g}{g_s} (D')^i(g,\bar{Q}). \label{eq:IS_representation_components_cond_var}
\end{align}

\paragraph{Dealing with the randomness of the estimation target.}
We now turn to second challenge. Since $\sigma_{0,t}^2$ can be written in terms of $\Phi_{0,i}(g_t,\widehat{\bar{Q}}_{t-1})$ for $i=1,2$, 
\cref{eq:IS_representation_components_cond_var} suggests perhaps an approach based on sample averages of $(g_t / g_s) (D')^i(g_t,\widehat{\bar{Q}}_{t-1})$ over $s$. However, whenever $s < t$, the latter is an $O(1),\ldots,O(t-1)$-measurable function due to the dependence on $g_t$ and $\widehat{\bar{Q}}_t$.
Namely, $P_{Q_0,g_s}\{ (g_t / g_s) (D')^i(g_t,\widehat{\bar{Q}}_{t-1})\}$ does not coincide in general with the conditional expectation $E_{Q_0,g_s} [((g_t / g_s) (D')^i(g_t,\widehat{\bar{Q}}_{t-1}))(O(s)) \mid \bar{O}(s-1)]$, as would arise from a sample average.
We now look at solutions to overcome this difficulty, considering first  $\widehat{\bar{Q}}_{t-1}$ and then $g_t$.

\subparagraph{Dealing with the randomness of $\widehat{\bar{Q}}_{t-1}$.} We propose an estimator of $\sigma_0^2(g,\widehat{\bar{Q}}_{t-1})$ for any fixed $g$. While requiring that $\widehat{\bar{Q}}_{t-1}$ converges to the true outcome regression function $\bar{Q}_0$ is a strong requirement, most reasonable estimators will at least converge to some fixed limit $\bar{Q}_1$. As a result, under an appropriate stochastic convergence condition on $(\widehat{\bar{Q}}_{t-1})_{t \geq 1}$, $\Phi_{0,i}(g,\widehat{\bar{Q}}_{t-1})$ can be reasonably approximated by the corresponding Cesaro averages, defined for $i=1,2$ as 
\begin{align}
    \bar{\Phi}_{0,i,t}(g) :=& \frac{1}{t-1} \sum_{s=1}^{t-1} \Phi_{0,i}(g,\widehat{\bar{Q}}_{s-1}) 
    = \frac{1}{t-1} \sum_{s=1}^t E_{Q_0,g_s} \left[ ((g / g_s) (D')^i(g,\widehat{\bar{Q}}_{s-1}))(O(s)) \mid \bar{O}(s-1)\right].
\end{align}
These are easy to estimate from the corresponding sample averages, defined for $i=1,2$ as 
\begin{align}
    \widehat{\Phi}_{i,t}(g) := \frac{1}{t-1} \sum_{s=1}^t ((g/g_s) (D')^i(g,\widehat{\bar{Q}}_{s-1}))(O(s)),
\end{align}
since for each $i=1,2$, the difference $\widehat{\Phi}_{i,t}(g) - \bar{\Phi}_{0,i,t}(g)$ is the average of a martingale difference sequence (MDS).
We then define our estimator of $\sigma_0^2(g, \widehat{\bar{Q}}_{t-1})$ as \begin{align}\label{eq:varest}\widehat{\sigma}_t(g) := \widehat{\Phi}_{1,t}(g) - (\widehat{\Phi}_{2,t}(g))^2.\end{align}

\subparagraph{From fixed $g$ to random $g_t$.} So far, we have proposed and justified the construction of $\widehat{\sigma}_t(g)$ as an estimator of $\sigma_{0,t}(g, \widehat{\bar{Q}}_{t-1})$ for a fixed $g$. We now discuss conditions under which $\widehat{\sigma}_t(g_t)$ is valid estimator of $\sigma_{0,t}(g_t, \widehat{\bar{Q}}_{t-1})$. When $g$ is fixed, for each $i=1,2$, the error $\widehat{\Phi}_{i,t}(g) - \Phi_{0,i,t}(g, \widehat{\bar{Q}}_{t-1})$ decomposes as the sum of the MDS average $\widehat{\Phi}_{i,t}(g) - \bar{\Phi}_{0,i,t}(g)$ and of the Cesaro approximation error $\bar{\Phi}_{0,i,t}(g) - \Phi_{0,i}(g, \widehat{\bar{Q}}_{t-1})$. Both differences are straightforward to bound. For a random $g_t$, the term $\widehat{\Phi}_{i,t}(g_t) - \bar{\Phi}_{0,i,t}(g_t)$ is no longer an MDS average. Fortunately, under a complexity condition on the logging policy class $\mathcal{G}$, we can bound the supremum of the martingale empirical processes $\{|\widehat{\Phi}_{i,t}(g) - \bar{\Phi}_{0,i,t}(g)| : g \in \mathcal{G} \}$, which in turn gives us a bound on $|\widehat{\Phi}_{i,t}(g_t) - \bar{\Phi}_{0,i,t}(g_t)|$.

\paragraph{Consistency guarantee for $\widehat{\sigma}^2_t$.} Our formal consistency result relies on the following assumptions.

\begin{assumption}[Outcome regression estimator convergence]\label{assumption:outcome_regression_convergence} There exists $\beta > 0$, and a fixed function $\bar{Q}_1 : \mathcal{A} \times \mathcal{X} \rightarrow \mathbb{R}$ such that $\|\widehat{\bar{Q}}_t - \bar{Q}_1 \|_{1,Q_{0,X}, g^*} = O(t^{-\beta})$ almost surely.
\end{assumption}

The next assumption is a bound on the bracketing entropy (see, \eg, \citep{vdV_Wellner96} for definition) of the logging policy class. 
\begin{assumption}[Complexity of the logging policy class]\label{assumption:logging_policy_entropy}
There exists a class of conditional densities $\mathcal{G}$ such that $g_t \in \mathcal{G}\,\forall t\geq1$ almost surely, there exists $G > 0$ such that $\sup_{g \in \mathcal{G}} \|g / g^\rf\|_\infty \leq G$, and for some $p > 0$
\begin{align}
    \log N_{[\,]}(\epsilon, \mathcal{G} / g^\rf, \|\cdot\|_{2,Q_{0,X},g^\rf}) \lesssim \epsilon^{-p},
\end{align}
where $\mathcal{G} / g^\rf := \{ g / g^\rf : g \in \mathcal{G} \}$.
\end{assumption}



Next, we require a condition on the exploration rate that is stronger than \cref{assumption:exp_rate_generic_stab_one_step}.
\begin{assumption}[Exploration rate (stronger)]\label{assumption:minimal_unif_exploration_rate}
For ant $t \geq 1$, we have that
$\inf_{a \in \mathcal{A}, x \in \mathcal{X}} g_t (a \mid x) / g^\rf(a \mid x) \gtrsim t^{-\alpha(\beta, p)}$ almost surely, where $\alpha(\beta, p) := \min(1/((3+p)), 1 / (1 + 2p), \beta)$.
\end{assumption}

\begin{theorem}\label{thm:sd}
Suppose that \cref{assumption:outcome_regression_convergence,assumption:logging_policy_entropy,assumption:minimal_unif_exploration_rate} hold. Then, $\widehat{\sigma}_t^2 - \sigma_{0,t}^2 = o(1)$ almost surely.
\end{theorem}

\begin{remark}
While we theoretically require the existence of a logging policy class $\mathcal{G}$ with controlled complexity, we do not actually need to know $\mathcal G$ to construct our estimator.
Moreover, while we require a bound on the bracketing entropy of the logging policy class $\mathcal{G}$, we impose no restriction on the outcome regression model complexity, permitting us to use flexible black-box regression methods.
\end{remark}

\begin{remark}
Assumption \ref{assumption:outcome_regression_convergence} requires $(\widehat{\bar{Q}}_t)$ to be a sequence of regression estimator, such that for every $t \geq 1$, $\widehat{\bar{Q}}_t$ is fitted on $O(1),\ldots,O(t)$ and for which we can guarantee a rate of convergence to some fixed limit $\bar{Q}_1$. Note that this can at first glance pose a challenge since observations $O(1),\ldots,O(t)$ are adaptively collected. 
In the appendix, we give guarantees for outcome regression estimation over a nonparametric model using an
importance sampling weighted empirical risk minimization.
\end{remark}

\paragraph{CADR asymptotics.}
Our proposed CADR estimator is now given by plugging our estimates $\widehat\sigma_t$ from \cref{eq:varest} into \cref{eq:stabilized-one-step}, as summarized in \cref{alg:cadr}
As an immediate corollary of \cref{thm:asymp_normal_stab_one_step,thm:sd} we have our main guarantee for this final estimator, showing CADR is asymptotically normal, whence we immediately obtain asymptotically valid confidence intervals.
\begin{corollary}[CADR Asymptotics and Inference]
Suppose that \cref{assumption:non_degenerate_EB,assumption:outcome_regression_convergence,assumption:logging_policy_entropy,assumption:minimal_unif_exploration_rate} hold.
Let $\widehat{\sigma}_t$ be given as in \cref{eq:varest}.
Denote $\Gamma_T := \left( T^{-1} \sum_{t=1}^T \widehat{\sigma}_t^{-1} \right)^{-1}$.
Then, $$\Gamma_T^{-1} \sqrt{T} \left( \widehat{\Psi}_T - \Psi_0 \right) \xrightarrow{d} \mathcal{N}(0,1).$$ Moreover, letting $\zeta_\alpha$ denote the $\alpha$-quantile of the standard normal distribution,
\begin{align}\notag
    \mathrm{Pr} \left[ \Psi(Q_{0,X}, \bar{Q}_0) \in \left[ \widehat{\Psi}_T  \pm {\zeta_{1-\alpha/2} \Gamma_T}/{\sqrt{T}}  \right] \right] \xrightarrow{T \to \infty} 1 - \alpha.
\end{align}
\end{corollary}


\section{Empirical Evaluation}
We next present computational results on public datasets that demonstrate the robustness of CADR confidence intervals using contextual bandit data with comparison to several baselines.
Our experiments focus on the case of finitely-many actions, $\A=\{1,\dots,K\}$.

\subsection{Baseline Estimators}
\label{sec:baselines}

We compare CADR to several benchmarks. All take the following form for a choice of $w_t,\omega_t,\widehat{\bar{Q}}_t$:
\begin{align}\notag
    \widehat{\Psi}_T &= \left( \frac{1}{T} w_t \right)^{-1} \frac{1}{T} \sum_{t=1}^T w_t \tilde D'_t,\quad \operatorname{CI}_\alpha=\left[\widehat\Psi_T\pm\zeta_{1-\alpha/2}\sqrt{\frac{\sum_{t=1}^Tw_t^2(\tilde D'_t-\widehat{\Psi})^2}{\left(\sum_{t=1}^Tw_t\right)^2}}\right],\\\text{where}~\tilde D'_t&=\omega_t(Y(t) - \widehat{\bar{Q}}_{t-1}(A(t),X(t))) + \sum_{a=1}^K \widehat{\bar{Q}}_{t-1}(a,X(t))g^*(a \mid X(t))
    .
\end{align}
The Direct Method (DM) sets $w_t=1,\omega_t=0$ and fits $\hat{\bar{Q}}_{t-1}(a,\cdot)$ by running some regression method for each $a$ on the data $\{(X(s),Y(s)):1 \leq s \leq t-1, A(s) = a\}$. We will use either linear regression or decision-tree regression, both using default \verb|sklearn| parameters.
Note that even in non-contextual settings, where $\hat{\bar{Q}}_{t-1}$ is a simple per-arm sample average, $\hat{\bar{Q}}_{t-1}$ may be biased due to adaptive data collection \citep{xu2013estimation,luedtke_vdL2016,bowden2017unbiased,nie2018adaptively,hadad2019confidence,shin2019bias}.
Inverse Propensity Score Weighting (IPW) sets $w_t=1,\omega_t=(g^*/g_t)(A(t) \mid X(t)),\hat{\bar{Q}}_{t}=0$. Doubly Robust (DR) sets $w_t=1,\omega_t=(g^*/g_t)(A(t) \mid X(t))$ and fits $\hat{\bar{Q}}_{t-1}$ as in DM. 
More Robust Doubly Robust (MRDR) \citep{farajtabar2018more} is the same as DR but when fitting $\hat{\bar{Q}}_{t-1}$ we reweight each data point by $\frac{g^*(A(s) | X(s))(1 - g_s(A(s) | X(s)))}{g_s(A(s) | X(s))^2}$.
None of the above are generally asymptotically normal under adaptive data collection \citep{hadad2019confidence}.
Adaptive Doubly Robust (ADR; a.k.a. stabilized one-step estimator for multi-armed bandit data) \citep{luedtke_vdL2016,hadad2019confidence} is the same as DR but sets $w_t=g^{-1/2}_t(A(t) | X(t))$. ADR is unbiased and asymptotically normal for multi-armed bandit logging policies but is biased for context-measurable adaptive logging policies, which is the focus of this paper. Finally, note that our proposal CADR takes the same form as DR but with $w_t=\widehat{\sigma}_t^{-1}$ using our adaptive conditional standard deviation estimators $\widehat{\sigma}_t$ in \cref{eq:varest}.

\subsection{Contextual Bandit Data from Multiclass Classification Data} \label{sec:multiclass} 
To construct our data, we turn $K$-class classification tasks into a $K$-armed contextual bandit problems \citep{dudik2014doubly,dimakopoulou2017estimation,su2019cab}, which has the benefits of reproducibility using public datasets and being able to make uncontroversial comparisons using actual ground truth data with counterfactuals. We use the public OpenML Curated Classification benchmarking suite 2018 (OpenML-CC18; BSD 3-Clause license) \citep{bischl2017openml}, which has datasets that
vary in domain, number of observations, number of classes and number of features. Among these, we select the classification datasets which have less than 100 features. This results in 57 classification datasets from OpenML-CC18 used for evaluation and \cref{tab:openml} summarizes the characteristics of these datasets.

Each dataset is a collection of pairs of covariates $X$ and labels $L\in\{1,\dots,K\}$. We transform each dataset to the contextual bandit problem as follows. At each round, we draw $X(t),L(t)$ uniformly at random with replacement from the dataset. We reveal the context $X(t)$ to the agent, and given an arm pull $A(t)$, we draw and return the reward $Y(t) \sim \mathcal{N}(\textbf{1}\{A(t) = L(t)\}, 1)$. To generate our data, we set $T=10000$ and use the following $\epsilon$-greedy procedure. We pull arms uniformly at random until each arm has been pulled at least once. Then at each subsequent round $t$, we fit $\widehat{\bar Q}_{t-1}$ using the data up to that time in the same fashion as used for the DM estimator above using decision-tree regressions. We set $\tilde A_x(t)=\argmax_{a=1,\dots,K}\widehat{\bar Q}_{t-1}(a,X(t))$ and $\epsilon_t=0.01 \cdot t^{-1/3}$. We then let $g_t(a\mid x)=\epsilon_t/K$ for $a\neq \tilde A_x(t)$ and $g_t(\tilde A_x(t)\mid x)=1-\epsilon_t+\epsilon_t/K$. That is, with probability $\epsilon_t$ we pull a random arm, and otherwise we pull $\tilde A_{X(t)}(t)$.

\begin{table}[t!]
\centering
\begin{tabular}{|c|c|}
\hline
Samples & Count \\
\hline
$< 1000$  & 17  \\ 
\hline
$\geq 1000$ and $< 10000$ & 30  \\ 
\hline
$\geq 10000$ & 10 \\
\hline
\end{tabular}
\hspace{5pt}
\begin{tabular}{|c|c|}
\hline
Classes & Count \\
\hline  
$= 2$ & 31 \\
\hline
$> 2 \text{ and } < 10$  & 17  \\ 
\hline
$ \geq 10 $ & 9  \\
\hline
\end{tabular}
\hspace{5pt}
\begin{tabular}{|c|c|}
\hline
Features & Count \\
\hline  
$\geq 2 \text{ and } < 10$ & 14 \\
\hline
$\geq 10 \text{ and } < 50$  & 34  \\ 
\hline
$\geq 50 \text{ and } \leq 100$ & 9  \\
\hline
\end{tabular}
\vspace{5pt}
\caption{Characteristics of the 57 OpenML-CC18 datasets used for evaluation.}
\label{tab:openml}
\end{table}
\begin{figure*}[t!]
\centering
\includegraphics[width=1\linewidth]{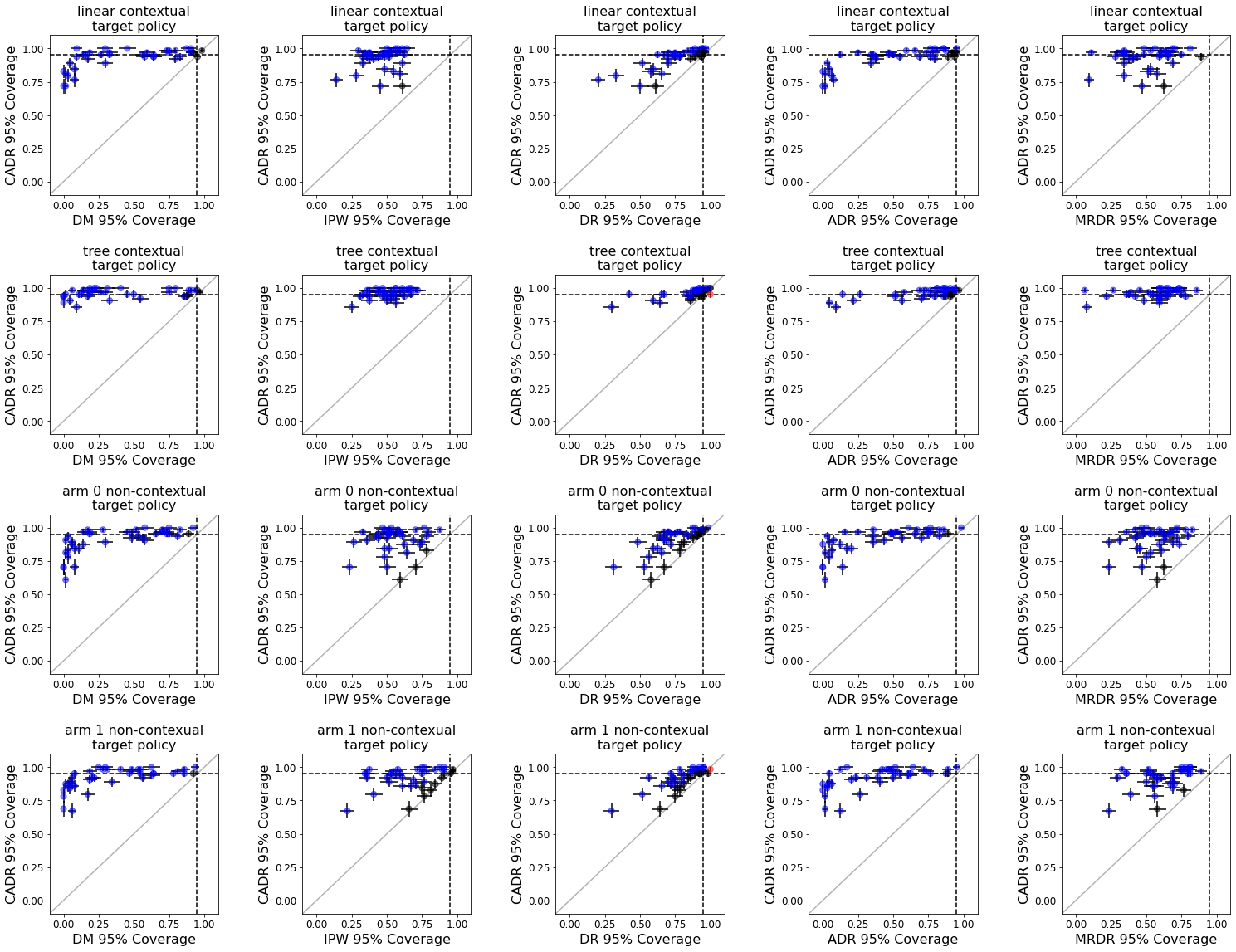}
\caption{Comparison of CADR estimator against DM, IPW, DR, ADR and MRDR w.r.t. 95\% confidence interval coverage on 57 OpenML-CC18 datasets and 4 target policies.}
\label{fig:coverage}
\end{figure*}

We then consider four candidate policies to evaluate:
(1) ``arm 1 non-contextual": $g^*(1 \mid x) = 1$ and otherwise $g^*(a \mid x) = 0$ (note that the meaning of label ``1'' changes by dataset),
(2) ``arm 2 non-contextual": $g^*(2 \mid x) = 1$ and otherwise $g^*(a \mid x) = 0$,
(3) ``linear contextual": we sample a \emph{new} dataset of size $T$ using a uniform exploration policy, then fit $\widehat{\bar Q}_{T}$ as above using linear regression, fix $a^*=\argmax_{a \in \{1, \dots, K\}}\widehat{\bar Q}_{T}(a,x)$, and set $g^*(a^*\mid x)=1$ and otherwise $g^*(a \mid x) = 0$,
(4) ``tree contextual": same as ``linear contextual" but fit $\widehat{\bar Q}_{T}$ using decision-tree regression.

\subsection{Results}

\Cref{fig:coverage} shows the comparison of CADR estimator against DM, IPW, DR, ADR, and MRDR w.r.t. coverage, that is, the frequency over 64 replications of the 95\% confidence interval covering the true $\Psi_0$, for each of the 57 OpenML-CC18 datasets and 4 target policies. In each subfigure, each dot represents a dataset, the $y$-axis corresponds to the coverage of the CADR estimator and the $x$-axis corresponds to the coverage of one of the baseline estimators. The lines represent one standard error over the 64 replications. The dot is depicted in blue if for that dataset CADR has significantly better coverage than the baseline estimator, in red if it has significantly worse coverage, and in black if the difference in coverage of both estimators is within one standard error. 
In \cref{fig:coverage}, outcome models for CADR, DM, DR, ADR, and MRDR are fit using linear regression (with default \verb|sklearn| parameters). In the appendix, we provide additional empirical results where we use decision-tree regressions, or where we use the MRDR outcome model for CADR, or where we use cross-fold estimation across time.

Across all of our experiments,
we observe that the confidence interval of CADR has better coverage of the ground truth than any other baseline, which can be attributed to its asymptotic normality. The second best estimator in terms of coverage is DR. The advantages of CADR over DR are most pronounced when either (a) there is a mismatch between the logging policy and the target policy (\eg, compare the 1st and 2nd rows in \cref{fig:coverage}; the tree target policy is most similar to the logging policy, which also uses trees) or (b) when the outcome model is bad (either due to model misspecification such as with a linear model on real data or due to small sample size).

\section{Conclusions}

Adaptive experiments hold great promise for better, more efficient, and even more ethical experiments. However, they complicate post-experiment inference, which is a cornerstone of drawing credible conclusions from controlled experiments.
We provided here the first asymptotically normal estimator for policy value and causal effects when data were generated from a contextual adaptive experiment, such as a contextual bandit algorithm. This led to simple and effective confidence intervals given by adding and subtracting multiples of the standard error, making contextual adaptive experiments a more viable option for experimentation in practice.

\section{Societal Impact and Limitations}

Adaptive experiments hold particular promise in settings where experimentation is costly and/or dangerous, such as in medicine and policymaking. By adapting treatment allocation, harmful interventions can be avoided, outcomes for study participants improved, and smaller studies enabled. Being able to draw credible conclusions from such experiments make them viable replacements for classic randomized trials. Our confidence intervals offer one way to do so. At the same time, and especially subject to our assumption of vanishing but nonzero exploration, these experiments must be subject to the same ethical guidelines as classic randomized experiments. Additionally, the usual caveats of frequentist confidence intervals hold here, such as its interpretation only as a guarantee over data collection, this guarantee only being approximate in finite samples when we rely on asymptotic normality, and the risks of multiple comparisons and of $p$-hacking. Finally, we note that our inference focused on an \emph{average} quantity, as such it focuses on social welfare and need not capture the risk to individuals or groups. Subgroup analyses may therefore be helpful in complementing the analysis; these can be conducted by setting $g^*(a\mid x)$ to zero for some $x$'s. Future work may be necessary to further extend our results to conducting inference on risk metrics such as quantiles of outcomes.

\bibliography{main}

\begin{thebibliography}{34}
\providecommand{\natexlab}[1]{#1}
\providecommand{\url}[1]{\texttt{#1}}
\expandafter\ifx\csname urlstyle\endcsname\relax
  \providecommand{\doi}[1]{doi: #1}\else
  \providecommand{\doi}{doi: \begingroup \urlstyle{rm}\Url}\fi

\bibitem[Athey et~al.(2018)Athey, Baird, Jamison, McIntosh, \"Ozler, and
  Sama]{atheytrial}
Susan Athey, Sarah Baird, Julian Jamison, Craig McIntosh, Berk \"Ozler, and
  Dohbit Sama.
\newblock A sequential and adaptive experiment to increase the uptake of
  long-acting reversible contraceptives in cameroon, 2018.
\newblock URL
  \url{http://pubdocs.worldbank.org/en/606341582906195532/Study-Protocol-Adaptive-experiment-on-FP-counseling-and-uptake-of-MCs.pdf}.
\newblock Study protocol.

\bibitem[Bakshy et~al.(2018)Bakshy, Dworkin, Karrer, Kashin, Letham, Murthy,
  and Singh]{bakshy2018ae}
Eytan Bakshy, Lili Dworkin, Brian Karrer, Konstantin Kashin, Benjamin Letham,
  Ashwin Murthy, and Shaun Singh.
\newblock Ae: A domain-agnostic platform for adaptive experimentation.
\newblock In \emph{Workshop on System for ML}, 2018.

\bibitem[Beygelzimer and Langford(2009)]{beygelzimer2009offset}
Alina Beygelzimer and John Langford.
\newblock The offset tree for learning with partial labels.
\newblock In \emph{Proceedings of the 15th ACM SIGKDD international conference
  on Knowledge discovery and data mining}, pages 129--138, 2009.

\bibitem[Bibaut et~al.(2021)Bibaut, Dimakopoulou, Chambaz, Kallus, and van~der
  Laan]{iswerm}
Aurelien Bibaut, Maria Dimakopoulou, Antoine Chambaz, Nathan Kallus, and Mark
  van~der Laan.
\newblock Risk minimization from adaptively collected data: Guarantees for
  supervised and policy learning.
\newblock 2021.

\bibitem[Bischl et~al.(2017)Bischl, Casalicchio, Feurer, Hutter, Lang,
  Mantovani, van Rijn, and Vanschoren]{bischl2017openml}
Bernd Bischl, Giuseppe Casalicchio, Matthias Feurer, Frank Hutter, Michel Lang,
  Rafael~G Mantovani, Jan~N van Rijn, and Joaquin Vanschoren.
\newblock Openml benchmarking suites.
\newblock \emph{arXiv preprint arXiv:1708.03731}, 2017.

\bibitem[Bowden and Trippa(2017)]{bowden2017unbiased}
Jack Bowden and Lorenzo Trippa.
\newblock Unbiased estimation for response adaptive clinical trials.
\newblock \emph{Statistical methods in medical research}, 26\penalty0
  (5):\penalty0 2376--2388, 2017.

\bibitem[Chernozhukov et~al.(2018)Chernozhukov, Chetverikov, Demirer, Duflo,
  Hansen, Newey, and Robins]{chernozhukov2017double}
Victor Chernozhukov, Denis Chetverikov, Mert Demirer, Esther Duflo, Christian
  Hansen, Whitney Newey, and James Robins.
\newblock Double/debiased machine learning for treatment and structural
  parameters.
\newblock \emph{The Econometrics Journal}, 21\penalty0 (1):\penalty0 C1--C68,
  2018.

\bibitem[Dimakopoulou et~al.(2017)Dimakopoulou, Zhou, Athey, and
  Imbens]{dimakopoulou2017estimation}
Maria Dimakopoulou, Zhengyuan Zhou, Susan Athey, and Guido Imbens.
\newblock Estimation considerations in contextual bandits.
\newblock \emph{arXiv preprint arXiv:1711.07077}, 2017.

\bibitem[Dud{\'\i}k et~al.(2011)Dud{\'\i}k, Langford, and Li]{dudik2011doubly}
Miroslav Dud{\'\i}k, John Langford, and Lihong Li.
\newblock Doubly robust policy evaluation and learning.
\newblock In \emph{Proceedings of the 28th International Conference on
  International Conference on Machine Learning}, pages 1097--1104, 2011.

\bibitem[Dud{\'\i}k et~al.(2014)Dud{\'\i}k, Erhan, Langford, Li,
  et~al.]{dudik2014doubly}
Miroslav Dud{\'\i}k, Dumitru Erhan, John Langford, Lihong Li, et~al.
\newblock Doubly robust policy evaluation and optimization.
\newblock \emph{Statistical Science}, 29\penalty0 (4):\penalty0 485--511, 2014.

\bibitem[Farajtabar et~al.(2018)Farajtabar, Chow, and
  Ghavamzadeh]{farajtabar2018more}
Mehrdad Farajtabar, Yinlam Chow, and Mohammad Ghavamzadeh.
\newblock More robust doubly robust off-policy evaluation.
\newblock In \emph{International Conference on Machine Learning}, pages
  1447--1456. PMLR, 2018.

\bibitem[Hadad et~al.(2019)Hadad, Hirshberg, Zhan, Wager, and
  Athey]{hadad2019confidence}
Vitor Hadad, David~A Hirshberg, Ruohan Zhan, Stefan Wager, and Susan Athey.
\newblock Confidence intervals for policy evaluation in adaptive experiments.
\newblock \emph{arXiv preprint arXiv:1911.02768}, 2019.

\bibitem[Kallus(2018)]{kallus2018balanced}
Nathan Kallus.
\newblock Balanced policy evaluation and learning.
\newblock In \emph{Advances in Neural Information Processing Systems}, pages
  8895--8906, 2018.

\bibitem[Kallus and Udell(2020)]{kallus2020dynamic}
Nathan Kallus and Madeleine Udell.
\newblock Dynamic assortment personalization in high dimensions.
\newblock \emph{Operations Research}, 68\penalty0 (4):\penalty0 1020--1037,
  2020.

\bibitem[Kallus and Uehara(2019{\natexlab{a}})]{kallus2019efficiently}
Nathan Kallus and Masatoshi Uehara.
\newblock Efficiently breaking the curse of horizon in off-policy evaluation
  with double reinforcement learning.
\newblock \emph{arXiv preprint arXiv:1909.05850}, 2019{\natexlab{a}}.

\bibitem[Kallus and Uehara(2019{\natexlab{b}})]{kallus2019intrinsically}
Nathan Kallus and Masatoshi Uehara.
\newblock Intrinsically efficient, stable, and bounded off-policy evaluation
  for reinforcement learning.
\newblock \emph{Advances in neural information processing systems}, 32,
  2019{\natexlab{b}}.

\bibitem[Kasy and Sautmann(2021)]{kasy2021adaptive}
Maximilian Kasy and Anja Sautmann.
\newblock Adaptive treatment assignment in experiments for policy choice.
\newblock \emph{Econometrica}, 89\penalty0 (1):\penalty0 113--132, 2021.

\bibitem[Li et~al.(2010)Li, Chu, Langford, and Schapire]{li2010contextual}
Lihong Li, Wei Chu, John Langford, and Robert~E Schapire.
\newblock A contextual-bandit approach to personalized news article
  recommendation.
\newblock In \emph{Proceedings of the 19th international conference on World
  wide web}, pages 661--670, 2010.

\bibitem[Li et~al.(2011)Li, Chu, Langford, and Wang]{li2011unbiased}
Lihong Li, Wei Chu, John Langford, and Xuanhui Wang.
\newblock Unbiased offline evaluation of contextual-bandit-based news article
  recommendation algorithms.
\newblock In \emph{Proceedings of the fourth ACM international conference on
  Web search and data mining}, pages 297--306, 2011.

\bibitem[Luedtke and van~der Laan(2016)]{luedtke_vdL2016}
Alexander~R. Luedtke and Mark~J. van~der Laan.
\newblock {Statistical inference for the mean outcome under a possibly
  non-unique optimal treatment strategy}.
\newblock \emph{The Annals of Statistics}, 44\penalty0 (2):\penalty0 713 --
  742, 2016.
\newblock \doi{10.1214/15-AOS1384}.
\newblock URL \url{https://doi.org/10.1214/15-AOS1384}.

\bibitem[Nie et~al.(2018)Nie, Tian, Taylor, and Zou]{nie2018adaptively}
Xinkun Nie, Xiaoying Tian, Jonathan Taylor, and James Zou.
\newblock Why adaptively collected data have negative bias and how to correct
  for it.
\newblock In \emph{International Conference on Artificial Intelligence and
  Statistics}, pages 1261--1269. PMLR, 2018.

\bibitem[Qiang and Bayati(2016)]{qiang2016dynamic}
Sheng Qiang and Mohsen Bayati.
\newblock Dynamic pricing with demand covariates.
\newblock \emph{arXiv preprint arXiv:1604.07463}, 2016.

\bibitem[Quinn et~al.(2019)Quinn, Teytelboym, Kasy, Gordon, and
  Caria]{kasytrial}
Simon Quinn, Alex Teytelboym, Maximilian Kasy, Grant Gordon, and Stefano Caria.
\newblock A sequential and adaptive experiment to increase the uptake of
  long-acting reversible contraceptives in cameroon, 2019.
\newblock URL \url{https://www.socialscienceregistry.org/trials/3870}.
\newblock Study registration.

\bibitem[Robins et~al.(1994)Robins, Rotnitzky, and Zhao]{robins1994estimation}
James~M Robins, Andrea Rotnitzky, and Lue~Ping Zhao.
\newblock Estimation of regression coefficients when some regressors are not
  always observed.
\newblock \emph{Journal of the American statistical Association}, 89\penalty0
  (427):\penalty0 846--866, 1994.

\bibitem[Shin et~al.(2019)Shin, Ramdas, and Rinaldo]{shin2019bias}
Jaehyeok Shin, Aaditya Ramdas, and Alessandro Rinaldo.
\newblock On the bias, risk and consistency of sample means in multi-armed
  bandits.
\newblock \emph{arXiv preprint arXiv:1902.00746}, 2019.

\bibitem[Su et~al.(2019)Su, Wang, Santacatterina, and Joachims]{su2019cab}
Yi~Su, Lequn Wang, Michele Santacatterina, and Thorsten Joachims.
\newblock Cab: Continuous adaptive blending for policy evaluation and learning.
\newblock In \emph{International Conference on Machine Learning}, pages
  6005--6014. PMLR, 2019.

\bibitem[Tewari and Murphy(2017)]{tewari2017ads}
Ambuj Tewari and Susan~A Murphy.
\newblock From ads to interventions: Contextual bandits in mobile health.
\newblock In \emph{Mobile Health}, pages 495--517. Springer, 2017.

\bibitem[Thomas and Brunskill(2016)]{thomas2016data}
Philip Thomas and Emma Brunskill.
\newblock Data-efficient off-policy policy evaluation for reinforcement
  learning.
\newblock In \emph{International Conference on Machine Learning}, pages
  2139--2148. PMLR, 2016.

\bibitem[van~der Laan and Robins(2003)]{van2003unified}
Mark~J van~der Laan and James~M Robins.
\newblock \emph{Unified methods for censored longitudinal data and causality}.
\newblock Springer Science \& Business Media, 2003.

\bibitem[van~der Vaart and Wellner(1996)]{vdV_Wellner96}
A.~van~der Vaart and J.~Wellner.
\newblock \emph{Weak Convergence and Empirical Processes}.
\newblock Springer-Verlag New York, 03 1996.
\newblock ISBN 9781475725452.

\bibitem[van~der Vaart(2000)]{van2000asymptotic}
Aad~W van~der Vaart.
\newblock \emph{Asymptotic statistics}.
\newblock Cambridge university press, 2000.

\bibitem[van Handel(2011)]{vanHandel2011}
R.~van Handel.
\newblock On the minimal penalty for {M}arkov order estimation.
\newblock \emph{Probability Theory and Related Fields}, 150:\penalty0 709--738,
  2011.

\bibitem[Wang et~al.(2017)Wang, Agarwal, and Dud{\i}k]{wang2017optimal}
Yu-Xiang Wang, Alekh Agarwal, and Miroslav Dud{\i}k.
\newblock Optimal and adaptive off-policy evaluation in contextual bandits.
\newblock In \emph{International Conference on Machine Learning}, pages
  3589--3597. PMLR, 2017.

\bibitem[Xu et~al.(2013)Xu, Qin, and Liu]{xu2013estimation}
Min Xu, Tao Qin, and Tie-Yan Liu.
\newblock Estimation bias in multi-armed bandit algorithms for search
  advertising.
\newblock \emph{Advances in Neural Information Processing Systems},
  26:\penalty0 2400--2408, 2013.

\end{thebibliography}

\section*{Checklist}

\begin{enumerate}

\item For all authors...
\begin{enumerate}
  \item Do the main claims made in the abstract and introduction accurately reflect the paper's contributions and scope?
    \answerYes{}
  \item Did you describe the limitations of your work?
    \answerYes{}
  \item Did you discuss any potential negative societal impacts of your work?
    \answerYes{}
  \item Have you read the ethics review guidelines and ensured that your paper conforms to them?
    \answerYes{}
\end{enumerate}

\item If you are including theoretical results...
\begin{enumerate}
  \item Did you state the full set of assumptions of all theoretical results?
    \answerYes{}
	\item Did you include complete proofs of all theoretical results?
    \answerYes{}
\end{enumerate}

\item If you ran experiments...
\begin{enumerate}
  \item Did you include the code, data, and instructions needed to reproduce the main experimental results (either in the supplemental material or as a URL)?
    \answerYes{In the supplemental material with specifics in Section \ref{execution} of the supplemental material.}
  \item Did you specify all the training details (\eg, data splits, hyperparameters, how they were chosen)?
    \answerYes{In Section \ref{sec:multiclass}}
	\item Did you report error bars (\eg, with respect to the random seed after running experiments multiple times)?
    \answerYes{In all figures \ref{fig:coverage}-\ref{fig:cf72_wellOPE_CAMRDR}}.
	\item Did you include the total amount of compute and the type of resources used (\eg, type of GPUs, internal cluster, or cloud provider)?
    \answerYes{In Section \ref{execution} of supplemental material.}
\end{enumerate}

\item If you are using existing assets (\eg, code, data, models) or curating/releasing new assets...
\begin{enumerate}
  \item If your work uses existing assets, did you cite the creators?
    \answerYes{In Section \ref{sec:multiclass}}
  \item Did you mention the license of the assets?
    \answerYes{In Section \ref{sec:multiclass}}
  \item Did you include any new assets either in the supplemental material or as a URL?
    \answerNA{}
  \item Did you discuss whether and how consent was obtained from people whose data you're using/curating?
    \answerNA{}
  \item Did you discuss whether the data you are using/curating contains personally identifiable information or offensive content?
    \answerNA{}
\end{enumerate}

\item If you used crowdsourcing or conducted research with human subjects...
\begin{enumerate}
  \item Did you include the full text of instructions given to participants and screenshots, if applicable?
    \answerNA{}
  \item Did you describe any potential participant risks, with links to Institutional Review Board (IRB) approvals, if applicable?
    \answerNA{}
  \item Did you include the estimated hourly wage paid to participants and the total amount spent on participant compensation?
    \answerNA{}
\end{enumerate}

\end{enumerate}

\newpage
\appendix

\begin{center}\large Supplementary Material for:\\{\Large\bf
Post-Contextual-Bandit Inference}\\[1ex]\large Anonymous Author(s)
\end{center}

\section{Proof of the asymptotic normality of CADR}

\begin{proof}[Proof of theorem \ref{thm:asymp_normal_stab_one_step}]
Recalling the definition of our estimator, we have that
\begin{align}
    &\sqrt{T}(\widehat{\Psi}_T - \Psi(\bar{Q}_0, Q_{0,X}) ) \\
    =& \Gamma_T \frac{1}{\sqrt{T}} \sum_{t=1}^T \widehat{\sigma}_t^{-1} \left( \Psi(\widehat{\bar{Q}}_{t-1}, \widehat{Q}_{X,t-1}) - \Psi(\bar{Q}_0, Q_{0,X}) + D(g_t, \widehat{\bar{Q}}_{t-1}, \widehat{Q}_{X,t-1})(O(t)) \right) \\
    =& \Gamma_T  \frac{1}{\sqrt{T}} \sum_{t=1}^T \widehat{\sigma}_t^{-1} \left(  D(g_t, \widehat{\bar{Q}}_{t-1})(O(t)) -  P_{Q_0, g_t} D(g_t, \widehat{\bar{Q}}_{t-1})(O(t)) \right) \\
    =& \Gamma_T  \frac{1}{\sqrt{T}} \sum_{t=1}^T (\delta_{O(t)} - P_{Q_0,g_t}) \widehat{\sigma}_t^{-1} (D')(g_t, \widehat{\bar{Q}}_{t-1}),
\end{align}
where 
\begin{align}
    (D')(g, \bar{Q}) :=& D(g,\bar{Q}, Q_{0,X}) + \Psi(\bar{Q}, Q_{0,X}) \\
    =&\frac{g^*}{g} (\widetilde{y} - \bar{Q}) + \int g^*(a \mid \cdot) \bar{Q}(a, \cdot) d\mu_{\mathcal{A}}(a).
\end{align}
Note that 
\begin{align}
    &\Var_{Q_0, g_t}((D')(g_t, \widehat{\bar{Q}}_{t-1})(O(t)) \mid \bar{O}(t-1)) \\
    =& \Var_{Q_0, g_t}(D(g_t, \widehat{\bar{Q}}_{t-1}, \widehat{Q}_{X,t-1})(O(t)) \mid \bar{O}(t-1)) \\
    =& \sigma_{0,t}^2.
\end{align}
Let $Z_{t,T} := T^{-1/2} \widehat{\sigma}_t^{-1} ( \delta_{O(t)} - P_{Q_0,g_t}) (D')(g_t, \widehat{\bar{Q}}_{t-1})$.

Observe that $\{Z_{t,T} : t= 1,\ldots,T,\ T \geq 1 \}$ is a martingale triangular array where, for every $T\geq 1$, $t \in [T]$, $Z_{t,T}$ is $\bar{O}(t)$-measurable. We will apply a martingale central limit theorem for triangular arrays 
to prove that $\sum_{t=1}^T Z_{t,T} \xrightarrow{d} \mathcal{N}(0,1)$. This will hold if we can check that
\begin{itemize}
    \item the sum of conditional variances $V_T := \sum_{t=1}^T \Var_{Q_0,g_t}(Z_{t,T} \mid \bar{O}(t-1))$ converges in probability to 1,
    \item the Lindeberg condition is satisfied, that is, for any $\epsilon > 0$,
    \begin{align}
        \sum_{t=1}^T E[Z_{t,T}^2 \Ind(Z_{t,T} > \epsilon) \mid \bar{O}(t-1) ] \xrightarrow{p} 0.
    \end{align}
\end{itemize}

\paragraph{Convergence of the sum of conditional variances.} We have that 
\begin{align}
    V_T := \frac{1}{T} \sum_{t=1}^T \Var_{Q_0, g_t}(Z_{T,t} \mid \bar{O}(t-1) ) = \frac{1}{T} \sum_{t=1}^T \frac{\sigma_{0,t}^2}{\widehat{\sigma}_t^2} = 1 + \frac{1}{T} \sum_{t=1}^T \frac{\sigma_{0,t}^2 - \widehat{\sigma}_t^2}{\sigma_{0,t}^2 + (\sigma_{0,t}^2 - \widehat{\sigma}_t^2)}.
\end{align}
We now show that the terms of the right-hand side of the last equality above are $o(1)$ a.s. As $\sigma_{0,t}^2 - \widehat{\sigma}_t^2 = o(1)$ a.s. by assumption, it suffices to show that $\sigma_{0,t}$ is lower bounded by a positive constant. 

For any fixed $Q_X$, $\bar{Q}$, $g$, we have that, $D(g, \bar{Q}, Q_X) = D(g, \bar{Q}_0, Q_{0,X}) + ( D(g, \bar{Q}, Q_X) - D(g, \bar{Q}_0, Q_{0,X}) )$. It is straightforward to check that $D(g, \bar{Q}_0, Q_{0,X})$ lies in the Hilbert space
$T_1(Q_0) := L_2^0(Q_{0,Y}) \oplus L_2^0(Q_{0,X})$,
where
\begin{align}
    L_2^0(Q_{0,Y}) :=& \left\lbrace h : \mathcal{O} \to \mathbb{R} : \forall (x, a) \in \mathcal{X} \times \mathcal{A},\ \int h(x,a,y) dQ_{0,Y}(y \mid a,x) = 0 \right\rbrace,\\
    \text{ and } L_2^0(Q_{0,X}) :=& \left\lbrace h : \mathcal{X} \to \mathbb{R} : \int h(x) dQ_{0,X}(x) = 0\right\rbrace,
\end{align}
while
$D(g, \bar{Q}, Q_X) - D(g, \bar{Q}_0, Q_{0,X})$ lies in the Hilbert space 
\begin{align}
    T_2(g) := L_2^0(g) := \left\lbrace h: \mathcal{X} \times \mathcal{A} \to \mathbb{R} : \forall x \in \mathcal{X},\ \int h(x,a) g(a \mid x) d\mu_{\mathcal{A}}(a) = 0 \right\rbrace.
\end{align}
It is straightforward to check that $T_1(Q_0)$ and $T_2(g)$ are orthogonal subspaces of $L_2(P_{Q_0,g})$.
We have
\begin{align}
    \sigma_0^2(g, \bar{Q}) 
    =& \left\lVert D(g,\bar{Q}, Q_X) \right\rVert_{2,Q_0,g}^2 - \left(P_{Q_0,g} D(g,\bar{Q}, Q_X)\right)^2 \\
    \geq & \left\lVert D(g,\bar{Q}, Q_X) \right\rVert_{2,Q_0,g}^2 \\
    =& \left\lVert D(g,\bar{Q}_0, Q_{0,X}) \right\rVert_{2,Q_0,g}^2 +\left\lVert D(g,\bar{Q}, Q_X) - D(g,\bar{Q}_0, Q_{0,X}) \right\rVert_{2,Q_0,g}^2 \\
    \geq & \left\lVert D(g,\bar{Q}_0, Q_{0,X}) \right\rVert_{2,Q_0,g}^2.
\end{align}
where we have used in the third line above that $D(g,\bar{Q}_0, Q_{0,X})$ and $D(g,\bar{Q}, Q_X) - D(g,\bar{Q}_0, Q_{0,X})$ lie in the orthogonal subspaces $T_1(Q_0)$ and $T_2(g)$. Therefore, 
\begin{align}
    \inf_{t \geq 1} \sigma_{0,t}^2 :=& \inf_{t \geq 1}\sigma_0^2(g_t, \widehat{\bar{Q}}_{t-1}) \\
    \geq & \inf_g \left\lVert D(g,\bar{Q}_0, Q_{0,X}) \right\rVert_{2,Q_0,g}^2\\
    >& 0,
\end{align}
where the last inequality is exactly assumption \ref{assumption:non_degenerate_EB}.

Therefore,
\begin{align}
     \left\lvert\frac{\sigma_{0,t}^2 - \widehat{\sigma}_t^2}{\sigma_{0,t}^2 + (\sigma_{0,t}^2 - \widehat{\sigma}_t^2)} \right\rvert \leq   \frac{\left\lvert \sigma_{0,t}^2 - \widehat{\sigma}_t^2 \right\rvert}{\inf_{s \geq 1} \sigma_{0,s}^2 + o(1)} = o(1)
\end{align}
almost surely. Therefore, by Cesaro summation, $V_T - 1 = o(1)$ a.s.

\paragraph{Checking Lindeberg's condition.} Let $\epsilon > 0$. We want to show that 
\begin{align}
    \sum_{t=1}^T E[Z_{t,T}^2 \Ind(Z_{t,T} \geq \epsilon)] \xrightarrow{p} 0. 
\end{align}
Let $\delta_t = \int_{a \in \mathcal{A}, x \in \mathcal{X}} g_t(a \mid x)$. From assumption \ref{assumption:exp_rate_generic_stab_one_step}, $\delta_t \gtrsim t^{-1/2}$/
We have that $Z_{t,T} O(\delta_t^{-1} T^{-1/2} \widehat{\sigma}_t^{-1} )$. Notice that $\widehat{\sigma}_t^{-1} = ( \sigma_{0,t}^2 + \widehat{\sigma}_t^2 - \sigma_{0,t}^2)^{-1/2} = O(1)$ a.s. since $\sigma^2_{0,t} \geq C  >0$ and $\widehat{\sigma}_t^2 - \sigma_{0,t}^2 = o(1)$. Therefore, $Z_{t,T} = O(\delta_t^{-1} T^{-1/2}) = o(1)$ a.s. since $\delta_t^{-1} = o(t^{-1/2})$ a.s., and therefore, almost surely, there exists $T_0(\epsilon)$ such that, for any $T \geq T_0(\epsilon)$, all the terms in the sum of the Lindeberg condition, are zero, which implies the the sum converges to zero almost surely.

Therefore, from the central limit theorem for martingale triangular arrays, 
\begin{align}
    \sqrt{T} \Gamma_T^{-1} ( \widehat{\Psi}_T - \Psi_0)  \xrightarrow{d} \mathcal{N}(0,1).
\end{align}
\end{proof}

\section{Estimation of $\sigma_{0,t}^2$ via sequential importance sampling.}

\subsection{Errors decomposition}

In the following lemma, we provide a useful decomposition of the IS-weighted integrands that appear in the expressions of $\Phi_{0,1}(g, \bar{Q})$ and $\Phi_{0,2}(g, \bar{Q})$.
\begin{lemma}\label{lemma:cond_var_integrands_decomp}
It holds that
\begin{align}
    \frac{g}{g_s} D_1^2(g,\bar{Q}) =& \frac{g^\rf}{g_s} f_1(g,\bar{Q}) + \frac{g^\rf}{g_s} f_2(\bar{Q}) + \frac{g^\rf}{g_s} f_3(g,\bar{Q}) \\
    \text{and } \frac{g}{g_s} D_1(g,\bar{Q}) =& \frac{g^\rf}{g_s} f_4(\bar{Q}) + \frac{g^\rf}{g_s} f_5(g, \bar{Q}),
\end{align}
where
\begin{align}
    f_1(g,\bar{Q}) :=& \frac{(g^* / g^\rf)^2}{(g / g^\rf)}(\widetilde{y} - \bar{Q})^2,\\
    f_2(\bar{Q}) :=& 2 (g^*/g^\rf) (\widetilde{y} - \bar{Q}) \int g^*(a \mid \cdot) \bar{Q}(a, \cdot) d\mu_{\mathcal{A}}(a),\\
    f_3(g,\bar{Q}) :=& (g / g^\rf) \left(\int g^*(a \mid \cdot) \bar{Q}(a, \cdot) d\mu_{\mathcal{A}}(a) \right)^2,\\
    f_4(\bar{Q}) :=& (g^* / g^\rf) (\widetilde{y} - \bar{Q}),\\
    f_5(g, \bar{Q}) :=& (g / g^\rf) \int g^*(a \mid \cdot) \bar{Q}(a, \cdot) d\mu_{\mathcal{A}}(a).
\end{align}
\end{lemma}

The decomposition above motivates the following definitions.
\begin{align}
    \widehat{\Phi}_{1,t}^{(1)}(g) :=& \frac{1}{t-1} \sum_{s=1}^{t-1} \delta_{O(s)} \frac{g^\rf}{g_s} f_1(g,\widehat{\bar{Q}}_{s-1}),\\
    \widehat{\Phi}_{1,t}^{(2)} :=& \frac{1}{t-1} \sum_{s=1}^{t-1} \delta_{O(s)} \frac{g^\rf}{g_s} f_2(\widehat{\bar{Q}}_{s-1}),\\
    \widehat{\Phi}_{1,t}^{(3)}(g) :=& \frac{1}{t-1} \sum_{s=1}^{t-1} \delta_{O(s)} \frac{g^\rf}{g_s} f_3(g, \widehat{\bar{Q}}_{s-1}),\\
    \widehat{\Phi}_{2,t}^{(1)} :=& \frac{1}{t-1} \sum_{s=1}^{t-1} \delta_{O(s)} \frac{g^\rf}{g_s} f_4(\widehat{\bar{Q}}_{s-1}),\\
    \widehat{\Phi}_{2,t}^{(2)}(g) :=& \frac{1}{t-1} \sum_{s=1}^{t-1} \delta_{O(s)} \frac{g^\rf}{g_s} f_5(g, \widehat{\bar{Q}}_{s-1}),
\end{align}
and
\begin{align}
    \bar{\Phi}_{0,1,t}^{(1)}(g) :=& \frac{1}{t-1} \sum_{s=1}^{t-1} P_{Q_0,g_s} \frac{g^\rf}{g_s} f_1(g,\widehat{\bar{Q}}_{s-1}),\\
    \bar{\Phi}_{0,1,t}^{(2)} :=& \frac{1}{t-1} \sum_{s=1}^{t-1} P_{Q_0,g_s} \frac{g^\rf}{g_s} f_2(\widehat{\bar{Q}}_{s-1}),\\
    \bar{\Phi}_{0,1,t}^{(3)}(g) :=& \frac{1}{t-1} \sum_{s=1}^{t-1} P_{Q_0,g_s} \frac{g^\rf}{g_s} f_3(g, \widehat{\bar{Q}}_{s-1}),\\
    \bar{\Phi}_{0,2,t}^{(1)} :=& \frac{1}{t-1} \sum_{s=1}^{t-1} P_{Q_0,g_s} \frac{g^\rf}{g_s} f_4(\widehat{\bar{Q}}_{s-1}),\\
    \bar{\Phi}_{0,2,t}^{(2)}(g) :=& \frac{1}{t-1} \sum_{s=1}^{t-1} P_{Q_0,g_s} \frac{g^\rf}{g_s} f_5(g, \widehat{\bar{Q}}_{s-1}),
\end{align}
and 
\begin{align}
    \Phi_{0,1}^{(1)}(g, \widehat{\bar{Q}}_{t-1}) :=& P_{Q_0,g_t} \frac{g^\rf}{g_s} f_1(g,\widehat{\bar{Q}}_{t-1}),\\
    \Phi_{0,1}^{(2)}(\widehat{\bar{Q}}_{t-1}) :=&  P_{Q_0,g_t} \frac{g^\rf}{g_s} f_2(\widehat{\bar{Q}}_{t-1}),\\
   \Phi_{0,1}^{(3)}(g, \widehat{\bar{Q}}_{t-1}) :=& P_{Q_0,g_t} \frac{g^\rf}{g_s} f_3(g, \widehat{\bar{Q}}_{t-1}),\\
    \Phi_{0,2}^{(1)}(\widehat{\bar{Q}}_{t-1}) :=&  P_{Q_0,g_t} \frac{g^\rf}{g_s} f_4(\widehat{\bar{Q}}_{t-1}),\\
    \Phi_{0,2}^{(3)}(g,\widehat{\bar{Q}}_{t-1}) :=& P_{Q_0,g_t} \frac{g^\rf}{g_s} f_5(g, \widehat{\bar{Q}}_{t-1}),
\end{align}
We have that 
\begin{align}
    \widehat{\Phi}_{1,t} =& \widehat{\Phi}_{1,t}^{(1)} + \widehat{\Phi}_{1,t}^{(2)} + \widehat{\Phi}_{1,t}^{(3)}, \text{ and } \widehat{\Phi}_{2,t} = \widehat{\Phi}_{2,t}^{(1)} + \widehat{\Phi}_{2,t}^{(2)},\\
    \bar{\Phi}_{0,1,t}(g) =&  \bar{\Phi}_{0,1,t}^{(1)}(g) + \bar{\Phi}_{0,1,t}^{(2)} + \bar{\Phi}_{0,1,t}^{(3)}(g), \text{ and } \bar{\Phi}_{0,2,t}(g) =  \bar{\Phi}_{0,1,t}^{(1)} + \bar{\Phi}_{0,2,t}^{(2)}(g),\\
    \Phi_{0,1}(g,\widehat{\bar{Q}}_{t-1}) =& \Phi_{0,1}^{(1)}(g,\widehat{\bar{Q}}_{t-1}) + \Phi_{0,1}^{(2)}(\widehat{\bar{Q}}_{t-1}) + \Phi_{0,1}^{(3)}(\widehat{\bar{Q}}_{t-1}),\\
    \Phi_{0,2}(g,\widehat{\bar{Q}}_{t-1}) =& \Phi_{0,2}^{(1)}(\widehat{\bar{Q}}_{t-1}) + \Phi_{0,2}^{(2)}(g,\widehat{\bar{Q}}_{t-1}).
\end{align}

We recall the decomposition of the errors $\widehat{\Phi}_{i,t}(g_t) - \Phi_{0,i}(g_t),\widehat{\bar{Q}}_{t-1})$ in a martingale empirical process term and an approximation term:
\begin{align}
    \widehat{\Phi}_{i,t}(g_t) - \Phi_{0,i}(g_t),\widehat{\bar{Q}}_{t-1}) =& (\widehat{\Phi}_{i,t}(g_t) - \bar{\Phi}_{0,i,t}(g_t)) + (\bar{\Phi}_{0,i,t}(g_t) - \Phi_{0,i}(g_t,\widehat{\bar{Q}}_{t-1})).
\end{align}

We treat the approximation terms in subsection \ref{subsection:approximation_errors} further down. We further decompose the martingale empirical process terms here. We have that 
\begin{align}
    \widehat{\Phi}_{1,t}(g_t) - \bar{\Phi}_{0,1,t}(g_t) =& (\widehat{\Phi}^{(1)}_{1,t}(g_t) - \bar{\Phi}^{(1)}_{0,1,t}(g_t)) + (\widehat{\Phi}^{(2)}_{1,t} - \bar{\Phi}^{(2)}_{0,1,t}) \\
    &+ (\widehat{\Phi}^{(3)}_{1,t}(g_t) - \bar{\Phi}^{(3)}_{0,1,t}(g_t)),\\
    \text{ and }  \widehat{\Phi}_{2,t}(g_t) - \bar{\Phi}_{0,2,t}(g_t) =& (\widehat{\Phi}^{(1)}_{2,t} - \bar{\Phi}_{0,2,t}^{(1)}) + (\widehat{\Phi}^{(2)}_{2,t}(g_t) - \bar{\Phi}_{0,2,t}^{(2)}(g_t)).
\end{align}
The two differences $\widehat{\Phi}^{(2)}_{1,t} - \bar{\Phi}^{(2)}_{0,1,t}$ and 
$\widehat{\Phi}^{(1)}_{2,t} - \bar{\Phi}_{0,2,t}^{(1)}$ are averages of martingale difference sequences, and can be analyzed with a martingale version of Bernstein's inequality. We bound the three other differences by the supremum of martingale empirical processes

\subsection{Control of the martingale empirical processes}

Let, for any $\delta > 0$, $\widetilde{\mathcal{G}}(\delta) := \{ g \in \mathcal{G} : \inf_{a,x} g(a \mid x) \geq \delta \}$. In the following lemma, we bound the sequential bracketing entropy of the classes of sequences of functions
\begin{align}
   \mathcal{F}_{k,t}(\delta) :=& \left\lbrace (f_1(g, \widehat{\bar{Q}}_{s-1}))_{s=1}^{t-1} : g \in \widetilde{\mathcal{G}}(\delta)  \right\rbrace,
\end{align}
for $k=1,3,5$.

\begin{lemma}[Sequential bracketing entropy bound]\label{lemma:sequential_bkting_ent_bound}
Suppose that assumption \ref{assumption:logging_policy_entropy} holds. Then, for $i = 3,5$,
\begin{align}
   \mathcal{N}_{[\,]}(\epsilon, \mathcal{F}_{1,t}(\delta), L_2(P_{Q_0, g^\rf})) \leq N_{[\,]}(G^{-2} \delta^2 \epsilon, \mathcal{G}, L_2(P_{Q_0, g^\rf})).
\end{align}
Suppose in addition that assumption \ref{assumption:minimal_unif_exploration_rate} also holds. For $k=3,5$, we then have that
\begin{align}
     \mathcal{N}_{[\,]}(\epsilon , \mathcal{F}_{k,t}(\delta), L_2(P_{Q_0,g^*})) \leq N_{[\,]}(\epsilon, \mathcal{G}, L_2(P_{Q_0, g^\rf})).
\end{align}
\end{lemma}

\begin{proof}[Proof of lemma \ref{lemma:sequential_bkting_ent_bound}]
Observe that 
\begin{align}
    0 \leq \int g^*(a \mid \cdot) \bar{Q}(a, \cdot) d\mu_{\mathcal{A}}(a) \leq 1,\qquad \text{and} \qquad 0 \leq & (g^* / g^\rf)^2 (\widetilde{y} - \bar{Q})^2 \leq G^2.
\end{align}
Let $\{(l^j, u^j) : j \in [N]\}$ be an $\epsilon$-bracketing of $\widetilde{\mathcal{G}}(\delta) / g^\rf$ in $L_2(P_{Q_0,g^\rf})$. Without loss of generality, we can assume that $u^j \geq l^j \geq \delta g^\rf$ for every $j$. Let $g \in \widetilde{\mathcal{G}}(\delta)$.  There exists $j$ such that $l^j \leq g \leq u^j$, and therefore,
\begin{align}
    f_1(u^j,\bar{Q}) \leq &  f_1(g, \bar{Q}) \leq f_1(l^j,\bar{Q}) \\
    \text{ and } f_k(l^j, \bar{Q}) \leq & f_k(g, \bar{Q}) \leq f_k(u^j, \bar{Q}), \text{ for } k=3,5.
\end{align}
We have that
\begin{align}
    &\left\lVert f_1(l^j,\bar{Q}) - f_1(u^j,\bar{Q})\right\rVert_{2,Q_0,g^\rf} \\
    = &\left\lVert (g^* / g^\rf) \frac{(u^j / g^\rf) - (l^j / g^\rf)}{(u^j / g^\rf) (l^j / g^\rf)} (\widetilde{y} - \bar{Q})^2\right\rVert_{2,Q_0,g^\rf} \\
    \leq & \delta^{-2} G^2 \epsilon
\end{align}
and for $k=3,5$, denoting $i_3 := 2$ and $i_5 := 1$, we have that
\begin{align}
    &\left\lVert f_k(u^j, \bar{Q}) - f_k(l^j, \bar{Q}) \right\rVert_{2,Q_0,g^\rf} \\
    =& \left\lVert ((u^j / g^\rf) - (l^j / g^\rf)) \int g^*(a \mid \cdot) \bar{Q}(a, \cdot) d\mu_{\mathcal{A}}(a) \right\rVert_{2,Q_0,g^\rf} \\
    \leq & \epsilon.
\end{align}
Therefore,
\begin{align}
    \rho((f_1(l^j,\widehat{\bar{Q}}_{s-1}) - f_1(u^j,\widehat{\bar{Q}}_{s-1}))_{s=1}^{t-1}) \leq \delta^{-2} G^2 \epsilon.
\end{align}
and, for $k=3,5$,
\begin{align}
    \rho((f_k(l^j,\widehat{\bar{Q}}_{s-1}) - f_k(u^j,\widehat{\bar{Q}}_{s-1}))_{s=1}^{t-1}) \leq \epsilon.
\end{align}
We have thus shown that an $\epsilon$-bracketing in $L_2(P_{Q_{0,X}, g^\rf})$ norm of $\mathcal{G} / g^\rf$ induces an $(G^2 \delta^{-1}, L_2(P_{Q_0, g^\rf}))$ sequential bracketing of $\mathcal{F}_{1,t}(\delta)$, and $(\epsilon, L_2(P_{Q_0, g^\rf}))$ sequential bracketings of $\mathcal{F}_{3,t}(\delta)$ and $\mathcal{F}_{5,t}(\delta)$, which yields the claims.
\end{proof}

\begin{lemma}[Uniform convergence of the martingale empirical process]\label{lemma:as_unif_conv_MEP}
Suppose that assumptions \ref{assumption:logging_policy_entropy} and \ref{assumption:minimal_unif_exploration_rate} hold. Then, for any $(i,j) \in \{(1,1), (1,3), (2,2)\}$
\begin{align}
    &\sup_{g \in \mathcal{G}} |\widehat{\Phi}_{i,t}^{(j)}(g) - \bar{\Phi}_{0,i,t}^{(j)}(g)| = o(1) \text{ a.s.}
\end{align}
\end{lemma}

\begin{proof}
Let $\delta := \min_{s \in [t-1]} \inf_{(a,x) \in \mathcal{A} \times \mathcal{X}} g_s(a \mid x)$. In this proof, we treat $G$ as a constant, and we absorb it in the symbols $\lesssim$, $O$, $o$, and $\widetilde{O}$ whenever we use them.

We treat the case $(i,j) = (1,1)$ and the case $(i,j) \in \{(1,3), (2,2)\}$ separately.

\paragraph{Case $(i,j) = (1,1)$.}
For any $g \in \mathcal{G}$, we have that $s \in [t-1]$, $\| f_1(g, \widehat{\bar{Q}}_{s-1}) \|\infty \leq G^2 \delta^{-1}$.
Therefore, from theorem \ref{thm:max_ineq_IS_mart_emp_proc}, for any $r^- \in (0, \delta^{-1} / 2]$, it holds with probability at least $1 - 2 e^{-x}$ that 
\begin{align}
    &\sup_{g \in \widetilde{\mathcal{G}}} \left\lvert \widehat{\Phi}_{1,t}^{(1)}(g) - \bar{\Phi}^{(1)}_{0,1,t}(g) \right\rvert \\
    \lesssim &  r^- + \frac{1}{\sqrt{\delta t}} \int_{r^-}^{G^2\delta^{-1}} \sqrt{\log (1 + \mathcal{N}_{[\,]}(\epsilon, \mathcal{F}_{1,t}(\delta), L_2(P_{Q_0, g^\rf})))} d\epsilon\\
    &+ \frac{G^2 \delta^{-1}}{\delta t} \log \mathcal{N}_{[\,]}(G^2 \delta^{-1}, \mathcal{F}_{1,t}(\delta), L_2(P_{Q_0, g^\rf})) \\
    &+ G^2 \delta^{-3/2} t^{-1/2} \sqrt{x} + G^2\delta^{-2} t^{-1} x.
\end{align}
Let $x_t := (\log t)^2$ and let $B_t$ the right-hand side above where we set $x$ to $x_t$. From Borel-Cantelli, we have that $\sup_{g \in \widetilde{\mathcal{G}}} |\widehat{\Phi}_{1,t}^{(1)}(g) - \bar{\Phi}^{(1)}_{0,1,t}(g)| = o(B_t)$ almost surely. Let us make $B_t$ explicit. 

From lemma \ref{lemma:sequential_bkting_ent_bound} and from assumption \ref{assumption:logging_policy_entropy}, we have that
\begin{align}
    & \frac{G^2 \delta^{-1}}{\delta t} \log (1 + \mathcal{N}_{[\,]}(G^2 \delta^{-1}, \mathcal{F}_{1,t}(\delta), L_2(P_{Q_0, g^\rf}))) \\
    \leq & \frac{G^2 \delta^{-1}}{\delta t}\log (1+N_{[\,]}(\delta, \mathcal{G} / g^\rf, L_2(P_{Q_0, g^\rf}))) \\
    \lesssim & G^2 \delta^{-(2+p)} t^{-1}.
\end{align}

Let us now focus on the entropy integral. We have that 
\begin{align}
    &\int_{r^-}^{G^2 \delta^{-1}} \sqrt{\log (1 + \mathcal{N}_{[\,]}( \epsilon, \mathcal{F}_{1,t}(\delta), L_2(P_{Q_0, g^\rf})))} d\epsilon \\
    \leq& \int_{r^-}^{G^2\delta^{-1}} \sqrt{\log (1 + N_{[\,]}(G^{-2} \delta^2\epsilon , \mathcal{G} / g^\rf, L_2(P_{Q_0, g^\rf})))} d \epsilon\\
    =& G^2\delta^{-2} \int_{G^{-2} \delta^2 r^- }^{\delta}  \sqrt{\log(1 + N_{[\,]}(u, \mathcal{G} / g^\rf, L_2(P_{Q_0, g^\rf})))} du \\
    =& G^2\delta^{-2} \int_{G^{-2} \delta^2 r^- }^{\delta} u^{-p/2} du \\
    =& \frac{G^2 \delta^{-2}}{1 - p/2} (\delta^{1-p/2} - (G^{-2} \delta^2 r^-)^{1-p/2},
\end{align}
for any $p \neq 2$. We choose $r^-$ so as to minimize the rate of $r^- + (\delta t)^{-1/2} \int_{r^-}^{G^2\delta^{-1}} \sqrt{\log (1 + \mathcal{N}_{[\,]}(\epsilon, \mathcal{F}_{1,t}(\delta), L_2(P_{Q_0, g^\rf})))} d\epsilon$. We distinguish the cases $p < 2$ and $p > 2$.

\subparagraph{Case $p < 2$.} We just set $r^- = 0$, and we obtain
\begin{align}
    r^- + (\delta t)^{-1/2}  \int_{r^-}^{G^2\delta^{-1}} \sqrt{\log (1 + \mathcal{N}_{[\,]}(\epsilon, \mathcal{F}_{1,t}(\delta), L_2(P_{Q_0, g^\rf})))} d\epsilon \lesssim \delta^{-\frac{1}{2}(3+p)} t^{-\frac{1}{2}}.
\end{align}
Collecting the other terms yields that $B_t = \widetilde{O}(\delta^{-(3+p)/2} t^{-1/2} + t^{-1} \delta^{-(2+p)})$. From assumption \ref{assumption:minimal_unif_exploration_rate}, $\delta \gtrsim t^{-\alpha}$, with $\alpha < \min(1/(3+p), 1/(1+2p))$, and we therefore have $B_t = o(1)$.

\subparagraph{Case $p > 2$.} We pick $r^-$ so as to balance both terms of $r^- + (\delta t)^{-1/2} \int_{r^-}^{G^2\delta^{-1}} \sqrt{\log (1 + \mathcal{N}_{[\,]}(\epsilon, \mathcal{F}_{1,t}(\delta), L_2(P_{Q_0, g^\rf})))} d\epsilon$. , that is we pick $r^-$ such that
\begin{align}
    r^- = t^{-1/2} G^p \delta^{-\frac{1}{2}(1 + 2 p)} \iff r^- = G^2 \delta^{-\frac{1}{p}(1 + 2 p)} t^{-\frac{1}{p}}.
\end{align}
Collecting the other terms then yields $B_t =\widetilde{O}(\delta^{-\frac{1}{p}(1 + 2 p)} t^{-\frac{1}{p}} + \delta^{-(2+p)} t^{-1})$. From assumption \ref{assumption:minimal_unif_exploration_rate}, $\delta \gtrsim t^{-\alpha}$, with $\alpha < \min(1/(3+p), 1/(1+2p))$, and we therefore have $B_t = o(1)$.

\paragraph{Case $(i,j) \in \{(1,3), (2,2)\}$.} For any $g \in \mathcal{G}$, $s \in [t-1]$, $k=3,5$, we have that $\|f_k(g,\widehat{\bar{Q}}_{s-1})\|_\infty \leq G$. Therefore, from theorem \ref{thm:max_ineq_IS_mart_emp_proc}, for any $(i,j,k) \in \{(1,3,3), (2,2,5)\}$, for any $x > 0$, it holds with probability at least $1 - 2 e^{-x}$ that 
\begin{align}
    \sup_{g \in \mathcal{G}} \left\lvert \widehat{\Phi}_{i,t}^{(j)} - \bar{\Phi}_{0,i,t}^{(j)} \right\rvert \lesssim & r^- + \frac{1}{\sqrt{\delta t}} \int_{r^-}^{G} \sqrt{\log (1+\mathcal{N}_{[\,]}(\epsilon, \mathcal{F}_{k,t}(\delta), L_2(P_{Q_0,g^\rf})))} d \epsilon \\
    &+ \frac{G}{\delta t} \log (1+\mathcal{N}_{[\,]}(G, \mathcal{F}_{k,t}(\delta), L_2(P_{Q_0,g^\rf}))) \\
    &+ G \sqrt{\frac{x}{\delta t}} + G \frac{x}{\delta t} \\
    \lesssim & r^- + \frac{1}{1-p/2} \frac{1}{\sqrt{\delta t}} ( G^{1-p/2} - (r^-)^{1-p/2} ) + \frac{G^{1-p}}{\delta t} + G \sqrt{\frac{x}{\delta t}} + G \frac{x}{\delta t},
\end{align}
where we have used that, from lemma \ref{lemma:sequential_bkting_ent_bound} and assumption \ref{assumption:logging_policy_entropy},  $\log (1+\mathcal{N}_{[\,]}(\epsilon, \mathcal{F}_{k,t}(\delta), L_2(P_{Q_0,g^\rf}))) \leq \log (1+N_{[\,]}(\epsilon, \mathcal{G} / g^\rf, L_2(P_{Q_0,g^\rf})) \lesssim \epsilon^{-p}$. Setting $x$ to $x_t := (\log t)^2$ in the bound above and denote $B_t$ the resulting quantity. Applying Borel-Cantelli's lemma yields that $\sup_{g \in \mathcal{G}} \left\lvert \widehat{\Phi}_{i,t}^{(j)} - \bar{\Phi}_{0,i,t}^{(j)} \right\rvert =o(B_t)$ almost surely. We now give an explicit bound on $B_t$.

\subparagraph{Case $p \in (0,2).$} We set $r^- = 0$. We obtain $B_t = \widetilde{O}((\delta t)^{-1/2} + (\delta t)^{-1})$. Since from assumption \ref{assumption:minimal_unif_exploration_rate}, $\delta \gtrsim t^{-\alpha}$ with $\alpha < 1$, we have that $B_t = o(1)$. 

\paragraph{Case $p > 2$.} We set $r^- = (\delta t)^{-1/p}$. We have $B_t = \widetilde{O}((\delta t)^{-1/p} + (\delta t)^{-1})$. Since from assumption \ref{assumption:minimal_unif_exploration_rate}, $\delta \gtrsim t^{-\alpha}$ with $\alpha < 1$, we have that $B_t = o(1)$. 
\end{proof}

\subsection{High probability bound for the martingale terms}

\begin{lemma}\label{lemma:high_proba_bound_mart_terms}
Suppose that there exists $\delta > 0$ such that $\|g^* / g_s\|_\infty \leq \delta^{-1}$ for every $s \in [t-1]$. Then
For $(i,j) \in \{(1,2), (2,1)\}$, for any $x > 0$, it holds with probability $1 - 2 e^{-x}$ that 
\begin{align}
    \left\lvert \widehat{\Phi}^{(j)}_{i,t} - \bar{\Phi}_{0,i,t}^{(j)} \right\rvert \lesssim \sqrt{\frac{x}{\delta t}} + \frac{x}{\delta t}, \label{eq:high_proba_bound_mart_terms_1}
\end{align}
and for $(i,j) \in \{ (1,3), (2,2) \}$, it holds with probability at least $1-2 e^{-x}$ that 
\begin{align}
    \left\lvert \widehat{\Phi}^{(j)}_{i,t} - \bar{\Phi}_{0,i,t}^{(j)} \right\rvert \lesssim \sqrt{\frac{x}{t}} + \frac{x}{t}.\label{eq:high_proba_bound_mart_terms_2}
\end{align}
\end{lemma}

\begin{proof}[Proof of lemma \ref{lemma:high_proba_bound_mart_terms}]
We have that 
\begin{align}
    \widehat{\Phi}^{(2)}_{1,t} - \bar{\Phi}_{0,1,t}^{(2)} =& \frac{1}{t-1} \sum_{s=1}^{t-1} (\delta_{O(s)} - P_{Q_0, g_s}) \frac{g^*}{g_s} f_2(\widehat{\bar{Q}}_{s-1})\\
    \text{and } \widehat{\Phi}^{(1)}_{2,t} - \bar{\Phi}_{0,2,t}^{(1)} =& \frac{1}{t-1} \sum_{s=1}^{t-1} (\delta_{O(s)} - P_{Q_0, g_s}) \frac{g^*}{g_s} f_4(\widehat{\bar{Q}}_{s-1}).
\end{align}
Therefore, both differences are the average of martingale difference sequences. For $k=2,4$, we have that $\|\frac{g^*}{g_s} f_k(\widehat{\bar{Q}}_{s-1}\|_\infty \leq \delta^{-1}$ and $\|\frac{g^*}{g_s} f_k(\widehat{\bar{Q}}_{s-1}\|_{2,Q_0,g^*} \leq \delta^{-1/2}$. Bernstein's inequality for martingale difference sequences then yields \eqref{eq:high_proba_bound_mart_terms_1}.

Concerning the other two differences, we have that
\begin{align}
    \widehat{\Phi}^{(3)}_{1,t} - \bar{\Phi}_{0,1,t}^{(3)} = \frac{1}{t-1} \sum_{s=1}^{t-1} Q_{0,X} f_3(\widehat{\bar{Q}}_{s-1})^2 \\
    \text{ and } \widehat{\Phi}^{(2)}_{2,t} - \bar{\Phi}_{0,2,t}^{(2)} = \frac{1}{t-1} \sum_{s=1}^{t-1} Q_{0,X} f_3(\widehat{\bar{Q}}_{s-1}).
\end{align}
These two terms are the average of martingale sequences too, and since $\|f_3(\widehat{\bar{Q}}_{s-1})\|_\infty \leq 1$, Bernstein's inequality for martingale difference sequences yields \eqref{eq:high_proba_bound_mart_terms_2}.
\end{proof}

\subsection{Approximation error lemma}\label{subsection:approximation_errors}

\begin{lemma}\label{lemma:approx_error_lemma}
For any  $\bar{Q}, \bar{Q}_1: \mathcal{A} \times \mathcal{X} \to \mathbb{R}$, it holds that 
\begin{align}
    \max \left\lbrace \left\lvert \Phi_{0,i}^{(j)}(\bar{Q}) - \Phi_{0,i}^{(j)}(\bar{Q}_1) \right\rvert : (i,j) \in \{(1,2),(1,3),(2,1),(2,2) \} \right\rbrace \leq 4 \left\lVert \bar{Q} - \bar{Q}_1 \right\rVert_{2,Q_0,g^*}
\end{align}
and for any conditional densities $(a,x) \mapsto g(a \mid x)$, and $(a,x) \mapsto g_1(a \mid x)$ such that $g_1, g \geq \delta$ for some $\delta > 0$, it holds that
\begin{align}
    \left\lvert \Phi_{0,1}^{(1)}(\bar{Q}) - \Phi_{0,1}^{(1)}(\bar{Q}_1) \right\rvert \leq \delta^{-2} \left\lVert g - g_1 \right\rVert_{1,Q_{0,X},g^*} + \delta^{-1} \left\lVert \bar{Q} - \bar{Q}_1 \right\rVert_{1, Q_{0,X}, g^*}.
\end{align}
\begin{proof}
We treat each case separately.
\paragraph{Case $(i,j) = (1,2)$.}
\begin{align}
    &\left\lvert \Phi_{0,1}^{(2)}(\bar{Q}) - \Phi_{0,1}^{(2)}(\bar{Q}_1) \right\rvert \\
    =& 2 \left\lvert P_{Q_0, g^*} \left\lbrace (\widetilde{y} - \bar{Q})\langle g^*, \bar{Q} \rangle - (\widetilde{y} - \bar{Q}_1)\langle g^*, \bar{Q}_1 \rangle \right\rbrace \right\rvert \\
    =& 2 \left\lvert P_{Q_0, g^*} \left\lbrace (\bar{Q}_1 - \bar{Q}) \langle g^*, \bar{Q} \rangle + (\widetilde{y} - \bar{Q}_1) \langle g^*, \bar{Q} - \bar{Q}) \rangle \right\rbrace \right\rvert \\
    \leq & 4 \left\lVert \bar{Q} - \bar{Q}_1 \right\rVert_{1, Q_{0,X}, g^*}
\end{align}

\paragraph{Case $(i,j) = (1,3)$.}
\begin{align}
    &\left\lvert \Phi_{0,1}^{(3)}(\bar{Q}) - \Phi_{0,1}^{(3)}(\bar{Q}_1) \right\rvert \\
    =& \left\lvert Q_{0,X} \left\lbrace \langle g^*, \bar{Q} \rangle^2 - \langle g^*, \bar{Q}_1 \rangle^2 \right\rbrace \right\rvert \\
    \leq & 2 Q_{0,X} \langle g^*, \left\lvert \bar{Q} - \bar{Q}_1 \right\rvert \rangle \\
    =& \left\lVert \bar{Q} - \bar{Q}_1 \right\rVert_{1,Q_{0,X},g^*}
\end{align}

\paragraph{Case $(i,j) = (2,1)$.}
\begin{align}
     &\left\lvert \Phi_{0,2}^{(1)}(\bar{Q}) - \Phi_{0,2}^{(1)}(\bar{Q}_1) \right\rvert \\
     =& \left\lvert P_{Q_0,g^*} \left\lbrace (\widetilde{y} - \bar{Q}) - (\widetilde{y} - \bar{Q}_1)  \right\rbrace \right\rvert  \\
     \leq & \left\lVert \bar{Q} - \bar{Q}_1 \right\rVert)_{1,Q_{0,X}, g^*}
\end{align}

\paragraph{Case $(i,j) = (2,2)$.} 
\begin{align}
    &\left\lvert \Phi_{0,2}^{(2)}(\bar{Q}) - \Phi_{0,2}^{(2)}(\bar{Q}_1)  \right\rvert \\
    =& \left\lvert Q_{0,X} \left\lbrace \langle g^*, \bar{Q} \rangle - \langle g^*, \bar{Q}_1 \rangle \right\rbrace \right\rvert \\
    \leq & \left\lVert \bar{Q} - \bar{Q}_1 \right\rVert_{1,Q_{0,X},g^*}
\end{align}

\paragraph{Case $(i,j) = (1,1)$.}
\begin{align}
    &\left\lvert \Phi_{0,1}^{(1)}(g, \bar{Q}) - \Phi_{0,1}^{(1)}(g_1,\bar{Q}_1)  \right\rvert \\
    =& \left\lvert P_{Q_0,g^*} \left\lbrace \frac{g^*}{g} (\widetilde{y} - \bar{Q}) - \frac{g^*}{g}(\widetilde{y} - \bar{Q}_1)\right\rbrace \right\rvert \\
    \leq & \left\lvert P_{Q_0, g^*} \left\lbrace \frac{1}{g g_1} (g - g_1) + \frac{1}{g} (\bar{Q} - \bar{Q}_1) \right\rbrace \right\rvert\\
    \leq & \frac{1}{\delta^2} \left\lVert g - g_1 \right\rVert_{1,Q_{0,X}, g^*} + \frac{1}{\delta} \left\lVert \bar{Q} - \bar{Q}_1 \right\rVert_{1,Q_{0,X},g^*}.
\end{align}
\end{proof}
\end{lemma}

\subsection{Proof of theorem \ref{thm:sd}}

\begin{proof}[Proof of theorem \ref{thm:sd}]
As noted at the beginning of this section, the estimation error $\widehat{\sigma}_t^2 - \sigma_{0,t}^2$ decomposes as
\begin{align}
    \widehat{\sigma}_t^2 - \sigma_{0,t}^2 :=& \sum_{(i,j) \in \mathcal{S}} \widehat{\Phi}_{i,t}^{(j)} - \bar{\Phi}_{0,i,t}^{(j)} \label{eq:var_est_decomp_eq1}\\ 
    &+ \sum_{(i,j) \in \mathcal{S}} \bar{\Phi}_{0,i,t}^{(j)} - \Phi_{0,i}^{(j)}(g_t, \bar{Q}_1) \label{eq:var_est_decomp_eq2} \\
    &+ \sum_{(i,j) \in \mathcal{S}} \Phi_{0,i}^{(j)}(g_t, \bar{Q}_1) - \Phi_{0,i}^{(j)}(g_t, \widehat{\bar{Q}}_{t-1}).\label{eq:var_est_decomp_eq3}
\end{align}
The terms in line \eqref{eq:var_est_decomp_eq1} are MDS averages or martingale empirical processes evaluated at $g_t$. Setting $x_t := (\log t)^2$ in lemma \ref{lemma:high_proba_bound_mart_terms} and using Borel-Cantelli gives that the MDS averages are $o(1)$ almost surely. Lemma \ref{lemma:as_unif_conv_MEP} gives that the martigale empirical process terms evaluated at $g_t$ are $o(1)$ almost surely as well.

From lemma \ref{lemma:approx_error_lemma}, and assumptions \ref{assumption:outcome_regression_convergence} and \ref{assumption:minimal_unif_exploration_rate}, 
\begin{align}
    &\sum_{(i,j) \in \mathcal{S}} \Phi_{0,i}^{(j)}(g_t, \bar{Q}_1) - \Phi_{0,i}^{(j)}(g_t, \widehat{\bar{Q}}_{s-1}) \\
    =& O(s^{\alpha -\beta}) \text{ a.s.} \\
    =& o(1) \text{ a.s.}.
\end{align}
Therefore the third line above \eqref{eq:var_est_decomp_eq3} is $o(1)$ almost surely, and by Ceasro summation, the second line above \eqref{eq:var_est_decomp_eq2} is $o(1)$ almost surely as well.
\end{proof}

\section{Maximal inequality for importance sampling weighted martingale empirical processes}

In this section, we restate a maximal inequality for so-called importance sampling martingale empirical processes from \citet{iswerm}. We include it for our reader's convenience.

\paragraph{Sequential bracketing entropy.}

Let $\Theta$ be a set, and let $T \geq 1$. For any $\theta \in \Theta$, let $(\xi_t(\theta))_{t=1}^T$ be a sequence of functions $\mathcal{O} \to \mathbb{R}$ such that for any $t \in [T]$, $\xi_t(\theta)$ is $\bar{O}(t-1)$-measurable. We denote
\begin{align}
    \Xi_T := \left\lbrace (\xi_t(\theta))_{t=1}^T : \theta \in \Theta \right\rbrace.
\end{align}

Let $g^\rf$ be a fixed reference policy. For any sequence $(f_t)_{t=1}^T$ of $\mathcal{O} \to \mathbb{R}$ functions such that $f_t$ is $\bar{O}(t-1)$-measurable for any $t$, we introduce the norm
\begin{align}
    \rho((f_t)_{t=1}^T) := \left( \frac{1}{T} \sum_{t=1}^T \|f_t\|_{2, Q_0, g^\rf}^2 \right)^{1/2}.
\end{align}

Following the definition of \cite{vanHandel2011}, we say that a collection of sequences of pairs of functions $\mathcal{O} \to \mathbb{R}$ of the form
\begin{align}
    \left\lbrace ((\lambda_t^j, \upsilon^j_t))_{t=1}^T : j \in [N] \right\rbrace
\end{align}
forms an $(\epsilon, L(P_{Q, g^\rf}))$ sequential bracketing of $\Xi_T$ if
\begin{itemize}
    \item for any $t \in [T]$ and any $j \in [N]$, $\lambda_t^j$ and $\upsilon_t^j$ are $\bar{O}(t-1)$-measurable $\mathcal{O} \to \mathbb{R}$ functions,
    \item for any $\theta \in [\Theta]$, there exists $j \in [N]$ such that, for any $t \in [T]$, $\lambda_t^j \leq \xi_t(\theta) \leq \upsilon_t^j$.
    \item for any $j \in [N]$, $\rho((\upsilon^j_t - \lambda^j_t)_{t=1}^T) \leq \epsilon$.
\end{itemize}
We denote $\mathcal{N}_{[\,]}(\epsilon, \Xi_T, L_2(P_{Q,g^\rf}))$ the cardinality of any $(\epsilon, L_2(P_{Q,g^\rf})$ sequential bracketing of $\Xi_T$ of minimal cardinality.

\paragraph{Importance sampling weighted martingale empirical process.} We term importance sampling weighting martingale empirical processes stochastic processes of the form
\begin{align}
    \left\lbrace \frac{1}{T} \sum_{t=1}^T (\delta_{O(t)} - P_{Q_0, g_t}) \frac{g^\rf}{g_t} \xi_t(\theta) : \theta \in \Theta \right\rbrace.
\end{align}

The result below is theorem 1 from {iswerm}.

\begin{theorem}[Maximal inequality for IS weighted martingale processes]\label{thm:max_ineq_IS_mart_emp_proc}
Suppose that 
\begin{itemize}
    \item there exists $\gamma > 0$ such that $\|g^* / g_t\|_\infty \leq \gamma$ for every $t \in [T]$,
    \item there exists $B > 0$ such that $\sup_{\theta \in \Theta} \|\xi_t(\theta)\|_\infty \leq B$ for every $t \in [T]$,
    \item there exists $p > 0$ such that 
    \begin{align}
        \log \mathcal{N}_{[\,]}(\epsilon, \Xi_T, L_2(P_{Q,g^\rf})) \lesssim \epsilon^{-p}.
    \end{align}
\end{itemize}
Then, for any $r > 0$, $r^- \in [0, r / 2]$ and $x > 0$,it holds with probability at least $1 - 2 e^{-x}$ that
\begin{align}
    &\sup_{g \in \mathcal{G}} \left\lbrace \frac{1}{T} \sum_{t=1}^T (\delta_{O(t)} - P_{Q_0, g_t}) \frac{g^\rf}{g_t} \xi_t(\theta) : \theta \in \Theta, \rho((\xi_t(\theta))_{t=1}^T) \leq \epsilon \right\rbrace \\
    \lesssim & r^- + \sqrt{\frac{\gamma}{T}} \int_{r^-}^r \sqrt{\log (1 + \mathcal{N}_{[\,]}(\epsilon, \Xi_T, P_{Q_0,g^\rf})} d\epsilon + \frac{\gamma B}{ T} \log (1 + \mathcal{N}_{[\,]}(r, \Xi_T, P_{Q_0,g^\rf}))\\
    &+  r\sqrt{\frac{\gamma x}{ T}} + \frac{\gamma B x}{ T}
\end{align}
\end{theorem}

\section{High probability bound for IS weighted nonparametric least squares from adaptively collected data}

Suppose $\mathcal{Y} \subseteq [-\sqrt{M}, \sqrt{M}]$ for some $M > 0$ and let $\bar{\mathcal{Q}}$ be a convex class of functions $\mathcal{A} \times \mathcal{X} \to \mathcal{Y}$. For any $\bar{Q}:\mathcal{A} \times \mathcal{X} \to \mathbb{R}$, and any $o = (x, a, y) \in \mathcal{O}$, let $\ell(\bar{Q},o) := (y - \bar{Q}(a,x))^2$. Let $g^\rf$ be a fixed (as opposed to random) density w.r.t. some dominating measure $\mu$ on $\mathcal{A}$. For any $\bar{Q}$, define the corresponding population risk w.r.t. $P_{Q_0, g^\rf}$ as $R_0(\bar{Q}) := P_{Q_0, g^\rf} \ell(\bar{Q}, \cdot)$. Observe that the population risk can be rewritten in terms of the conditional distributions $(P_{Q_0,g_s})_{s=1}^t$ of observations $(O(s))_{s=1}^t$ given their respective past, via IS weighting:
\begin{align}
    R_0(\bar{Q}) := \frac{1}{t} \sum_{s=1}^t P_{Q_0, g_s} \frac{g^\rf}{g_s} \ell(\bar{Q}, \cdot).
\end{align}
We define the corresponding IS weighted empirical risk as
\begin{align}
    \widehat{R}_t(\bar{Q}) :=  \frac{1}{t} \sum_{s=1}^t \delta_{O(s)} \frac{g^\rf}{g_s} \ell(\bar{Q}, \cdot).
\end{align}

Let $\widehat{\bar{Q}}_t \in \argmin_{\bar{Q} \in \bar{\mathcal{Q}}} \widehat{R}_t(\bar{Q})$ be an empirical risk minimizer over $\bar{\mathcal{Q}}$. In the upcoming theorem, we provide a high probability bound on the excess risk $R_0(\widehat{\bar{Q}}_t) - R_0(\bar{Q}_1)$. Our result requires the following assumptions.

\begin{assumption}[Entropy of the loss class]\label{assumption:ent_loss_class}
There exists $p > 0$ such that $\log N_{[\,]}(M \epsilon, \ell(\bar{\mathcal{Q}}), L_2(P_{Q_0, g^\rf})) \lesssim \epsilon^{-p}$, where $\ell(\bar{\mathcal{Q}}) := \{ \ell(\bar{Q}) : \bar{Q} \in \bar{\mathcal{Q}} \}.$
\end{assumption}

\begin{assumption}[Bounded IS ratios]\label{assumption:bounded_IS_ratios}
There exists $\gamma_t > 0$ such that $\|g^* / g_s\|_\infty \leq \gamma_t$ for every $s=1,\ldots,t$.
\end{assumption}

Theorem 4 in \citet{iswerm} gives a high probability excess risk bound on the least squares estimator. We restate it here under the current notation for our reader's convenience.

\begin{theorem}
Consider the setting of the current section, and suppose that \ref{assumption:ent_loss_class} and \ref{assumption:bounded_IS_ratios} hold. Then, for any $x > 0$, it holds with proability $1-2 e^{-x}$ that
\begin{align}
    R(\widehat{\bar{Q}}_t) - \inf_{\bar{Q} \in \bar{\mathcal{Q}}} R(\bar{Q}) \lesssim M \begin{cases}
    \left( \frac{\gamma_t}{t}\right)^{\frac{1}{1+p/2}} + 
    \frac{\gamma_t x}{t} 
    & \text{ if } p < 2, \\
    \left(\frac{\gamma_t}{t}\right)^{\frac{1}{p}} +  \frac{\gamma_t}{t} + \sqrt{\frac{ \gamma_t x}{t}} + \frac{\gamma_t x}{t}  
    & \text{ if } p > 2.
    \end{cases}
\end{align}
\end{theorem}



\section{Additional Empirical Results}
\label{apdx:empirical}
\subsection{Sequential Sample Splitting vs. Cross-Time-Fitting}
The approach we proposed in the main text estimates $\widehat{\bar{Q}}_{t-1}$ using only the data $O(1), \dots, O(t-1)$. This means that potentially few data are available for earlier estimates. In this section, we empirically explore an alternative strategy for fitting $\widehat{\bar{Q}}_{t-1}$ inspired by the cross-time-fitting procedure proposed in \citet{kallus2019efficiently} and which would be theoretically justified under some sufficient mixing (which is not necessary for our sequential approach).
Specifically, we split our data into $F=4$ folds 
and train $F$ outcome regression models, $\widehat{\bar{Q}}_f,\,f=1,2,3,4$, each to be used to make predictions on data in the corresponding fold. The model $\widehat{\bar{Q}}_f$ is trained using observations in all folds except for folds $f$ and $\min(f+1, F)$. As long as the data is sufficiently mixing, dropping fold $f+1$ ensures sufficient independence from future data. At the same time, each model now uses an amount of data that grows linearly in $T$.
Further, unlike sequential sample splitting, which requires training of $T-1$ models, cross-time-fitting requires training only $F$ models. \Cref{fig:sss57_misOPE_CADR,fig:cf57_misOPE_CADR} establish parity in the conclusions w.r.t. CADR's coverage compared to all other baseline estimators on 57 OpenML-CC18 datasets, 4 target policies and linear outcome regression models for all estimators that use them when these models are trained with sequential sample splitting (as in \Cref{fig:coverage} of the \cref{sec:multiclass} in the main text) and with time cross-fitting respectively.

\begin{figure*}[!ht]
\centering
\includegraphics[width=0.75\linewidth]{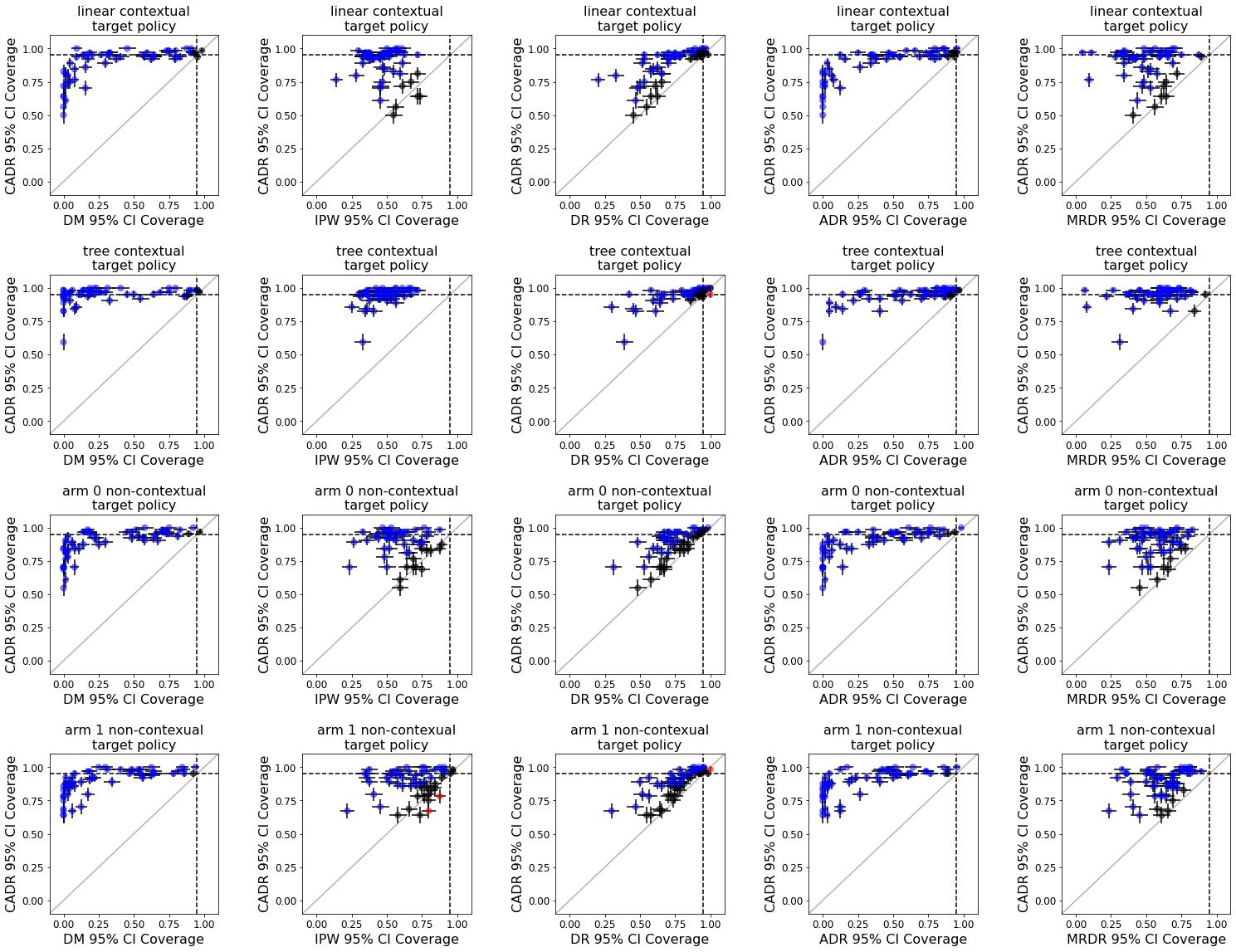}
\caption{Comparison of CADR estimator against DM, IPW, DR, ADR, MRDR w.r.t 95\% confidence interval coverage on 57 OpenML-CC18 datasets and 4 target policies with \textbf{sequential sample splitting} for training the linear outcome regression model of all estimators that use them.}
\label{fig:sss57_misOPE_CADR}
\end{figure*}

\begin{figure*}[!ht]
\centering
\includegraphics[width=0.75\linewidth]{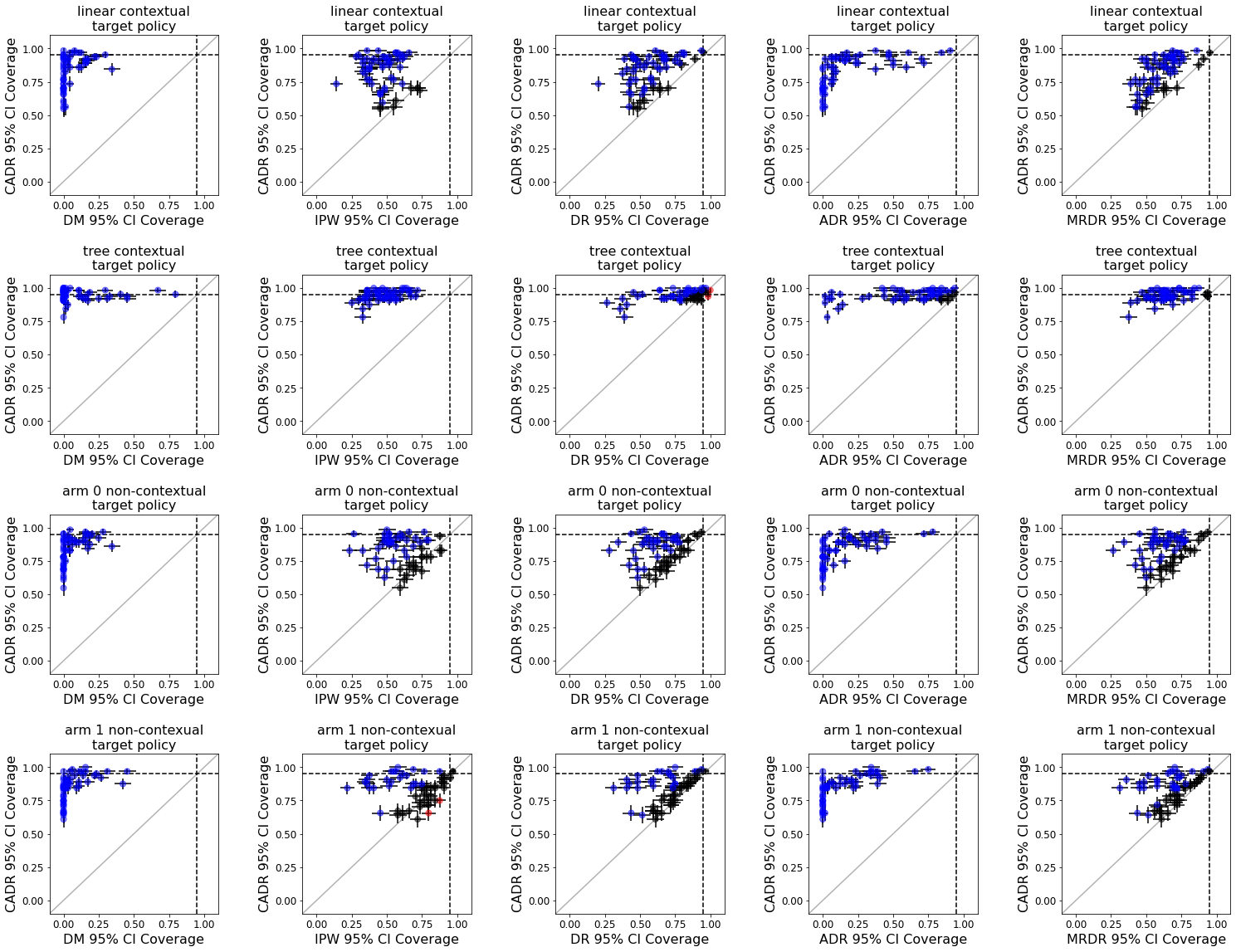}
\caption{Comparison of CADR estimator against DM, IPW, DR, ADR, MRDR w.r.t. 95\% confidence interval coverage on 57 OpenML-CC18 datasets and 4 target policies with \textbf{cross-fitting} for training the linear outcome regression model of all estimators that use them.}
\label{fig:cf57_misOPE_CADR}
\end{figure*}

\subsection{CADR in Misspecified vs. Well-Specified Outcome Regression Models}

Although CADR's advantage over DR is more pronounced when the off-policy estimator's outcome regression model is misspecified (\eg, using linear model on real data), this section establishes the advantage of CADR over all other estimators when they all use a well-specified outcome regression model (\eg, tree).  \Cref{fig:cf72_misOPE_CADR} shows CADR's coverage performance when the outcome regression model of DM, DR, MRDR and CADR is misspecified (linear regression model trained with the default \verb|sklearn| parameters) and \cref{fig:cf72_wellOPE_CADR} shows CADR's coverage performance when the outcome regression model of DM, DR, MRDR and CADR is well-specified (decision tree regression model trained with the default \verb|sklearn| parameters). Each dot represents each one of the 72 datasets and is colored blue when CADR has significantly better coverage than the corresponding baseline column estimator, in red when it has significantly worse coverage and in black when the two coverage are within standard error.  Results are averaged over 64 simulations per dataset and standard errors are shown. CADR remains the best estimator in both cases but as expected, in the misspecified outcome regression model case there are more datasets where CADR has significantly better coverage than DR compared to the well-specified outcome regression model case where there are more datasets for which CADR's and DR's coverage are within standard error. This is because when the error is large and is multiplied by a potentially large inverse propensity score of the logging policy, the variance stabilization performed by CADR is the most effective.

\begin{figure*}[!ht]
\centering
\includegraphics[width=0.75\linewidth]{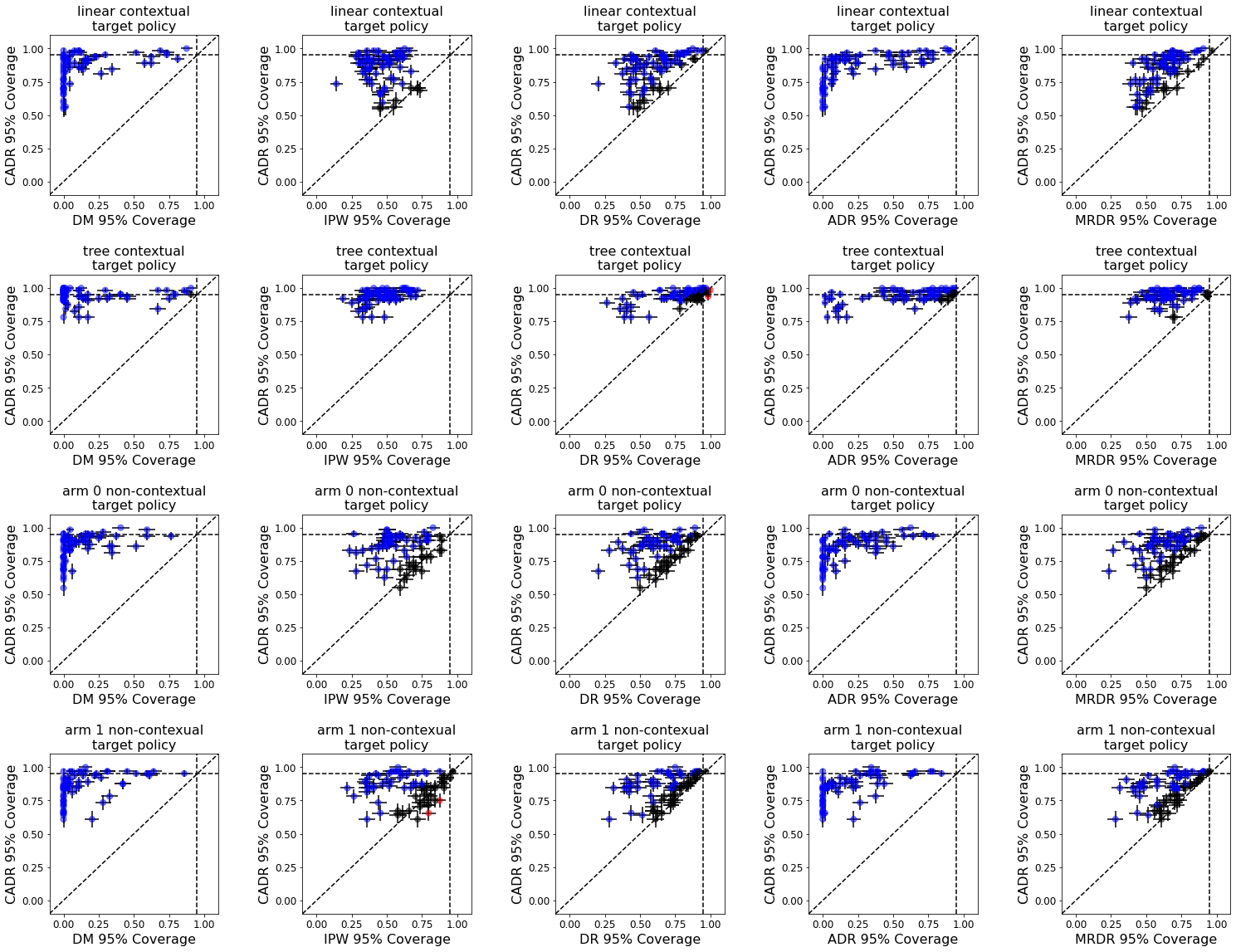}
\caption{Comparison of CADR estimator against DM, IPW, DR, ADR, MRDR w.r.t. 95\% confidence interval coverage on all 72 OpenML-CC18 datasets and 4 target policies with \textbf{linear outcome regression model (misspecified)} trained with cross-fitting of all estimators that use them.}
\label{fig:cf72_misOPE_CADR}
\end{figure*}

\begin{figure*}[!ht]
\centering
\includegraphics[width=0.75\linewidth]{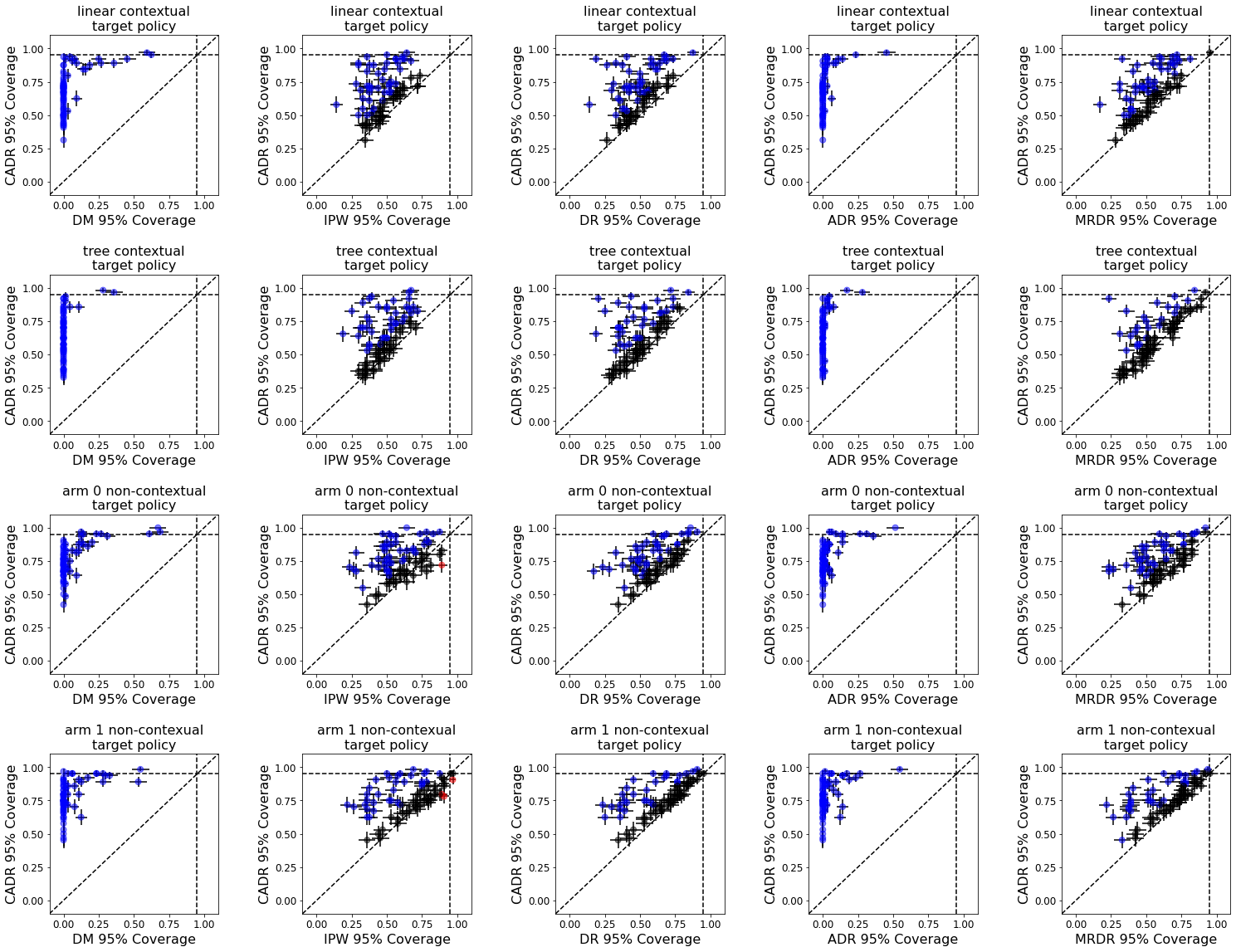}
\caption{Comparison of CADR estimator against DM, IPW, DR, ADR, MRDR w.r.t. 95\% confidence interval coverage on all 72 OpenML-CC18 datasets and 4 target policies with \textbf{tree outcome regression model (well-specified)} trained with cross-fitting of all estimators that use them.}
\label{fig:cf72_wellOPE_CADR}
\end{figure*}

\subsection{Importance Sampling Weighted Training of CADR Outcome Regression Model}

Finally, we consider the effect of using weighted training in the outcome model fitting of CADR akin to MRDR's outcome model fitting, where each training sample $O(s) = (X(s), A(s), Y(s))$ is weighted by $w(s) = \frac{g^*(A(s) | X(s))}{g_s(A(s) | X(s))}$. We call this estimator CAMRDR. \Cref{fig:cf72_misOPE_CAMRDR} shows CAMRDR's coverage performance against baselines and CADR when the outcome regression model of DM, DR, MRDR, CADR and CAMRDR is misspecified (linear regression model trained with the default \verb|sklearn| parameters). \Cref{fig:cf72_wellOPE_CAMRDR} shows CAMRDR's coverage performance against baselines and CADR when the outcome regression model of DM, DR, MRDR, CADR and CAMRDR is well-specified (decision tree regression model trained with the default \verb|sklearn| parameters). Again, each dot represents each one of the 72 datasets and is colored blue when CAMRDR has significantly better coverage than the corresponding column estimator, in red when it has significantly worse coverage and in black when the two coverage are within standard error.  Results are averaged over 64 simulations per dataset and standard errors are shown. Importance sampling weighted training makes a small positive difference compared to CADR in the well-specified case and a small negative difference compared to CADR in the mis-specified case. CAMRDR is better than all other baselines in both cases.

\begin{figure*}[!ht]
\centering
\includegraphics[width=0.80\linewidth]{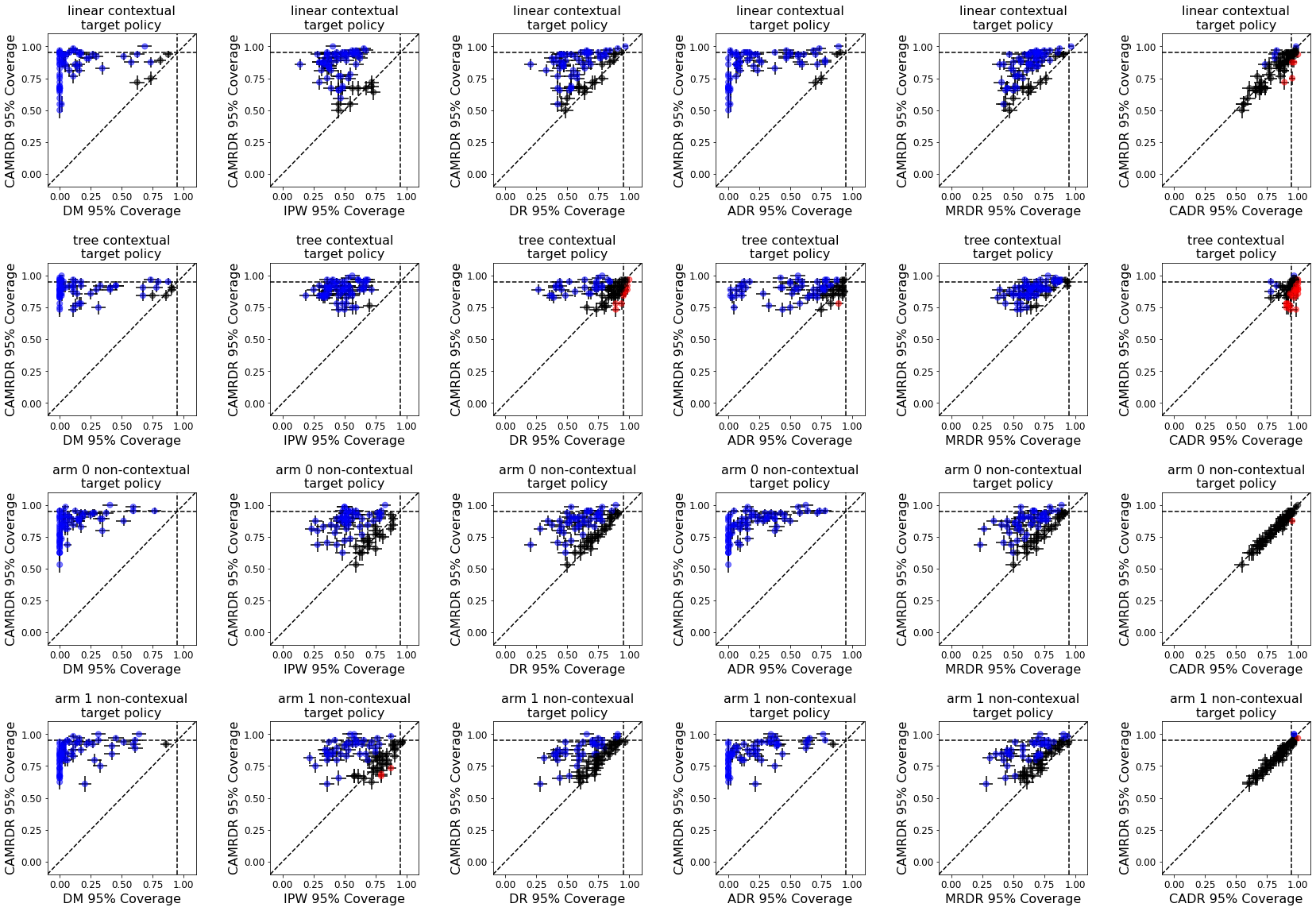}
\caption{Comparison of CAMRDR estimator against DM, IPW, DR, ADR, MRDR and CADR (last column)  w.r.t. 95\% confidence interval coverage on all 72 OpenML-CC18 datasets and 4 target policies with \textbf{linear outcome regression model (misspecified)} trained with cross-fitting.}
\label{fig:cf72_misOPE_CAMRDR}
\end{figure*}

\begin{figure*}[!ht]
\centering
\includegraphics[width=0.80\linewidth]{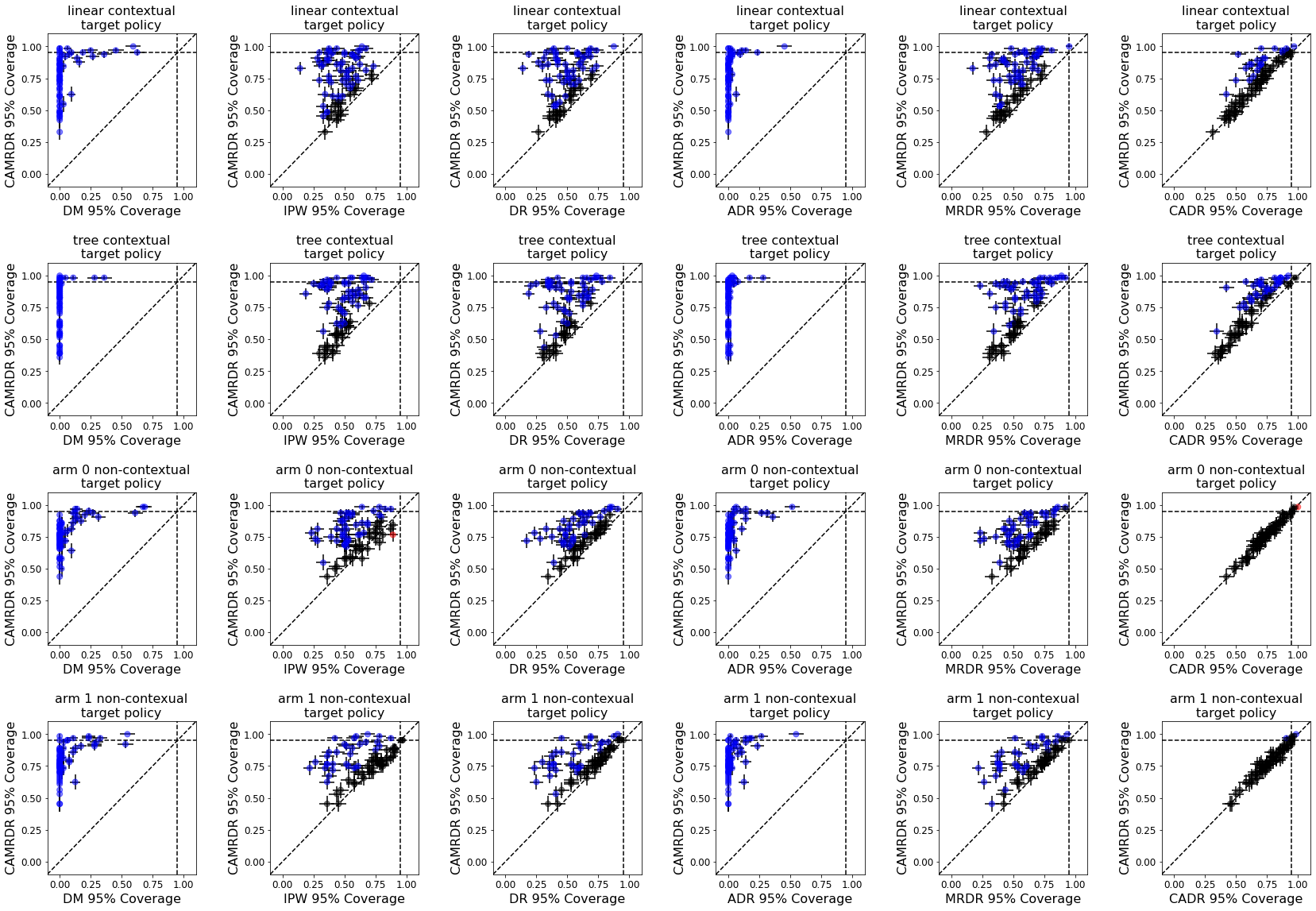}
\caption{Comparison of CAMRDR estimator against DM, IPW, DR, ADR, MRDR and CADR (last column) w.r.t. 95\% confidence interval coverage on all 72 OpenML-CC18 datasets and 4 target policies with \textbf{tree outcome regression model (well-specified)} trained with cross-fitting.}
\label{fig:cf72_wellOPE_CAMRDR}
\end{figure*}

\subsection{Execution Specifics of Experiment Code}
\label{execution}
The IPython notebook to reproduce the experimental results of the main paper and the appendix is included as an attachment in the supplemental materials. One needs to obtain an OpenML API key to run this code (instructions can be found at https://docs.openml.org/Python-guide/) and replace the string \verb|'YOURKEY'| in \verb|summarize_openmlcc18()| and in \verb|download_openmlcc18()| functions with it.
After that, if the notebook is executed as is, it reproduces Figure \ref{fig:cf57_misOPE_CADR} (1h 26min on a 64 CPU Intel Xeon). Changing variable  \verb|ope_outcome_model_training| from \verb|cross_fitting| to \verb|sequential_sample_splitting| reproduces Figures \ref{fig:coverage}/\ref{fig:sss57_misOPE_CADR} (same)  (22h 23min on a 64 CPU Intel Xeon). Changing variable   \verb|task_min_samples| from 1000 to 0 and variable \verb|task_max_contexts| to \verb|np.inf| reproduces Figure \ref{fig:cf72_misOPE_CADR} (20h 20min on a 64 CPU Intel Xeon).
Changing variable  \verb|ope_outcome_model| from \verb|LinearRegression()| to \verb|DecisionTreeRegressor()|, variable   \verb|task_min_samples| from 1000 to 0 and variable \verb|task_max_contexts| to \verb|np.inf| reproduces Figure \ref{fig:cf72_wellOPE_CADR} (26h 8min on a 64 CPU Intel Xeon).
Figures \ref{fig:cf72_misOPE_CAMRDR} and \ref{fig:cf72_wellOPE_CAMRDR} are from the same execution as Figures \ref{fig:cf72_misOPE_CADR} and \ref{fig:cf72_wellOPE_CADR} but with adding \verb|'CAMRDR'| in the \verb|competitors| variable of the \verb|visualize_coverage()| function. 

\end{document}